\documentclass[twocolumn]{autart}
\usepackage{setspace}
\usepackage{amssymb}
\usepackage{flushend}
\usepackage{bm}
\usepackage{placeins}
\usepackage{xcolor}
\usepackage{mathtools}
\usepackage{makecell}
\usepackage{booktabs}
\usepackage{multirow} 
\usepackage{array}
\usepackage{subcaption}
\usepackage{subfig}
\usepackage{textcomp}
\usepackage{stfloats}
\usepackage{dsfont}
\usepackage{url}
\usepackage{verbatim}
\usepackage{graphicx}
\usepackage{amsmath,amsfonts}
\usepackage{algorithmic}
\usepackage{algorithm}
\usepackage{graphicx}
\usepackage{cite}
\newtheorem{definition}{Definition}

\newtheorem{theorem}{Theorem}
\newtheorem{corollary}{Corollary}
\newtheorem{assumption}{Assumption}

\newtheorem{proposition}{Proposition}
\newtheorem{remark}{Remark}
\newenvironment{proof}[1][Proof]
{\par\noindent\textbf{#1.}\ }
{\hspace*{\fill}$\square$\par}
\DeclareMathOperator*{\argmin}{arg\,min}

\begin{document}
\begin{frontmatter}

\title{Convolutional Bayesian Filtering} 

\thanks[footnoteinfo]{This work is supported by Tsinghua University Education Foundation fund (042202008). This paper was not presented at any meeting. Corresponding author: Shengbo Eben Li.
}

\author[THU_SVM]{Wenhan Cao}\ead{cwh19@mails.tsinghua.edu.cn},
\author[THU_SVM]{Shiqi Liu}\ead{liushiqi19@gmail.com},
\author[PKU]{Chang Liu}\ead{changliucoe@pku.edu.cn},     \author[THU_SVM]{Zeyu He}\ead{hezeyu1549@gmail.com},
\author[THU_MATH,CAS_MATH]{Stephen S.-T. Yau}\ead{yau@uic.edu},  
\author[THU_SVM]{Shengbo Eben Li}\ead{lisb04@gmail.com},

\address[THU_SVM]{School of Vehicle and Mobility, Tsinghua University, Beijing 100084, China}                                             
\address[PKU]{Department of Advanced Manufacturing and Robotics, Peking University, Beijing 100871, China}             

\address[THU_MATH]{Department of Mathematical Sciences, Tsinghua University, Beijing 100084, China}        

\address[CAS_MATH]{Beijing Institute of Mathematical Sciences and Applications (BIMSA), Beijing 101408, China} 
          

\begin{abstract}                          
Bayesian filtering serves as the 
mainstream framework of state estimation in dynamic systems.
Its standard version utilizes total probability rule and Bayes' law alternatively, where how to define and compute conditional probability is critical to state distribution inference. 
Previously, the conditional probability is assumed to be exactly known, 
which represents a measure of the occurrence probability of one event, given the second event.
In this paper, we find
that by adding an additional event that stipulates
an inequality condition, we
can transform the conditional probability into a special integration that is analogous to convolution. 
Based on this transformation, we show that both transition
probability and output probability can be generalized to convolutional forms, 
resulting in a more general filtering framework that we call convolutional Bayesian filtering.  This new framework encompasses standard
Bayesian filtering as a special case when the distance metric of the inequality condition
is selected as Dirac delta function. It also allows for a more nuanced consideration of model mismatch by choosing different types of inequality conditions.
For instance, when the distance metric is defined in a distributional sense, the transition probability and output probability can be approximated by simply rescaling them into fractional powers. 
Under this framework, a robust version of Kalman filter can be constructed by only altering the noise covariance matrix, while maintaining the conjugate nature of Gaussian distributions.
Finally, we exemplify the effectiveness of our approach by reshaping classic filtering algorithms into convolutional versions, including Kalman filter, extended Kalman filter, unscented Kalman filter and particle filter.
\end{abstract}
\end{frontmatter}

\section{Introduction}\label{sec:introduction}

Accurately estimating the state value of dynamic systems is a crucial task in science and engineering, such as 
robotics, power systems, aerospace engineering, and manufacturing.  
Since the 1960s, Bayesian filtering has become a principled framework for optimal state estimation. The essence of this framework is to find a balance between uncertain system model and noisy state measurement. Its associated algorithm iteratively updates the probability density function of system state using the prior from the last step and the likelihood of the new observation. Afterward, typical estimation criteria like minimum mean-square error and maximum a posteriori are utilized to acquire the optimal point estimate.
    
In mathematics, Bayesian filtering relies on two conditional probabilities: transition probability and output probability. The transition probability describes how the system state evolves over time, and the output probability depicts the relationship between noisy measurement and ground truth state. To incorporate the information of those two probabilities, each iteration of Bayes filter is composed of two steps \cite{chen2003bayesian,distributedBayes}: prediction and update. 
The prediction step employs \textit{total probability rule} to integrate the product of transition probability and the state distribution of previous time to obtain the prior distribution.
The update step employs  \textit{Bayes' law}
to calculate the posterior distribution by adding information of the current measurement, where the output probability is used as a likelihood term.
Since being proposed, Bayesian filtering has become the foundation of optimal filtering algorithms, including the well-known Kalman filter family, particle filter, and variational Bayesian filter.

The origin of optimal filtering theory can be traced back to the early 1940s, marked by the groundbreaking contribution of Norbert Wiener \cite{ wiener1949extrapolation} and Kolmogorov \cite{kolmogorov1941interpolation}. This field reached a significant milestone in 1960 with the invention of discrete-time Kalman filter \cite{kalman1960new}, followed by its continuous-time version published
one year later \cite{kalman1961new}. Unlike Wiener's work, which deals with stationary processes in the frequency domain, Kalman filter addresses dynamic processes in the time domain. 
The Kalman filter is a direct consequence of applying Bayesian filtering to linear Gaussian systems. 
In its prediction step, the conditional  probability, i.e., transition probability of system model, is Gaussian. 
Due to the closure property of Gaussian distributions in linear transformation, the resulting expectation after applying total probability rule is also Gaussian, which is referred to as the prior distribution.
In the update step, the conditional probability, i.e., output probability of measurement model, is naturally Gaussian. Owing to the
conjugacy property of Gaussian distributions, the resulting posterior of Bayes' law keeps Gaussian. 
Since both the prior and posterior are proven to be Gaussian, Kalman filter can be solved analytically by solely computing the mean and variance of Gaussian distribution using closed-form formulas.

  
When facing nonlinear systems, one big issue of Bayesian filtering is that the closure and conjugacy properties no longer hold. As a result, the calculation of  total probability rule and Bayes' law in almost all nonlinear systems has no analytical solution. So far,
 several approximation methods have been proposed to replace the accurate calculation of the two rules. A notable early advancement in this area was extended Kalman filter (EKF), pioneered by NASA for spacecraft navigation \cite{smith1962application}. This extended filter employs 
 the Taylor series expansion to linearize nonlinear dynamics around the current state. Its state estimation can achieve the so-called first-order accuracy, i.e., EKF is perfect if the dynamics is linear with respect to the state. 
In contrast,  unscented Kalman filter (UKF) employs the unscented transform, a deterministic sampling technique, to achieve a third-order accuracy in approximating nonlinear dynamics for symmetric noise distributions \cite{julier2004unscented}.
This technique acquires the transformed distribution by generating a set of sigma points around the mean state estimate and then propagating them through the known nonlinear function, offering the advantage of preserving second-moment information compared to EKF. 
Obviously, both EKF and UKF  implicitly calculate the solution of total probability rule and Bayes' law using Gaussian distributions. Due to the adoption of approximation techniques, neither of them can offer formal guarantees on the estimation accuracy in highly nonlinear systems.
 
Instead of approximate  prior and posterior as Gaussian distributions, particle filter (PF) represents them with a group of particles by the Monte Carlo method \cite{liu1998sequential}. In the prediction step, particles are propagated according to the transition probability of system model to predict the distribution of next state. The resulting particles constitute a discrete approximation of the prior. In the update step, each particle receives a weight related to the output probability of observed data. All the particles are resampled according to their weights, which builds a discrete approximation of the posterior. It has been proven that these kinds of discrete distributions can converge to real distributions as the number of particles becomes sufficiently large. Nevertheless, PF often requires substantial computational resources due to the use of Monte carlo sampling, which limits its application in many real-time scenarios.

The variational Bayesian filter addresses the intensive computation associated with particles by adopting variational inference as an alternative approximation technique \cite{sarkka2009recursive, krishnan2017structured}. The prior of its prediction step is assumed  to be in a Gaussian form. This assumption is achieved by computing the expectation of a conditional probability using the closure of Gaussian distribution, which is identical to Kalman filter.
The update step avoids calculating the computationally intensive integral in the Bayes' law. Instead, it seeks to numerically minimize the Kullback-Leibler divergence between the proposal distribution and the real posterior. The proposal distribution is often chosen as a parameterized function to obtain an approximate solution of the minimization problem. The real posterior is represented as the product of the prior distribution and the output probability. 
In general, variational Bayesian filter is computationally beneficial in high-dimensional estimation problems and it has been widely adopted in adaptive filtering applications.

All the state estimation algorithms discussed above adhere strictly to the Bayesian filtering framework. In this framework, the mathematical form of conditional probabilities, including transition probability and output probability, plays an important rule in computing total probability rule and Bayes' law.  The standard definition of conditional probability is
a measure of the occurrence probability of one event, given
the second event. Moreover, its distribution in the whole space is often assumed to be exactly known in Bayes filter design.
In this paper, we find that by conditioning on an additional event, which stipulates a distance metric between two observed variables 
within a specified threshold, one can transform the conditional probability to a special integral form that is similar to convolution operation. This definition relaxes the  necessity of information completeness, which allows us to design a more general filter. 
We define this new probability as convolutional conditional probability. 

Based on this definition, the transition probability can be extended to a convolutional form by conditioning on the event that the real state and its virtual state satisfy an inequality condition. The same extension can be applied to the output probability. Collectively, these two extensions forge a generalized filtering framework, which we refer to as convolutional Bayesian filtering.
This new framework encompasses standard Bayesian filtering as a special case when the distance metric is set as the Dirac delta function. 
One of its natural benefit is the capability to explicitly handle mismatches between mathematical model and the real system by tailoring the distance metric properly.

Under this new framework, we can reformulate nearly all Bayesian filtering algorithms into a more generalized type. Particularly, convolutional Bayes filter possesses analytical forms of convolution operation in systems with Gaussian noises, which allows to design a robust Kalman filter family. For non-Gaussian systems, convolution operation usually has no analytical forms but can be efficiently
approximated by a newly proposed exponential density rescaling technique. This technique enables to rescale transition probability and output probability into their fractional powers when the distance metric is defined in a distributional sense.  
We further establish the theoretical connection between this approximation technique and information bottleneck theory. It is proven that the fractional power from density rescaling is related to Lagrange multiplier of an optimization problem whose objective is to modulate the balance between preserving information about measurement model and squeezing representation of measurement data.


The remainder of this paper is organized as follows: Section~\ref{sec.conditional probability framework} introduces the definition of convolutional conditional probability. Section~\ref{sec.convolutional Bayesian filtering} discusses the framework of convolutional Bayesian filtering.
Section~\ref{sec.theoretical approximation} introduces an approximation technique for non-Gaussian systems. 
Section~\ref{sec.simulations} shows the simulation results. Section~\ref{sec.discussion and conclusion}  concludes this paper.

\section{Convolutional Conditional Probability}\label{sec.conditional probability framework}
As discussed before, Bayesian filtering is built upon two pillars:
total probability rule \eqref{eq.total probability rule} and Bayes' law \eqref{eq.Bayes' law}. Both of them rely on how to handle the conditional probability $p(y|x)$, which is a measure of the occurrence probability of the event $\left\{\mathbf{y} = y\right\}$, given the event $\left\{\mathbf{x} = x\right\}$. 
This can also be interpreted as the ratio of the probability of both events happening to the probability of the original event.
According to this interpretation, the total probability rule is
\begin{subequations}\label{eq.rule}
\begin{equation}
p(y) = \int {p(y|x)}p(x) \, \mathrm{d}x,
\label{eq.total probability rule}
\end{equation}
\text{and the Bayes' law is}
\begin{equation}
p(x|y) = \frac{p(x){p(y|x)}}{\int p(x) {p(y|x)}\,\mathrm{d} x}.
\label{eq.Bayes' law}
\end{equation}
\end{subequations}
Note that we use boldface to denote a random variable, such as \( \mathbf{x} \), and normal font to denote the realization of this variable, such as \( x \).
Previously, the explicit form of the conditional probability $p(y|x)$ is assumed known directly in Bayes filter design. One may be interested in what will happen if the conditional probability $p(y|x)$ is unknown. This question helps us to conceive a new definition of conditional probability, i.e., convolutional conditional probability.
\begin{figure}[t]
\centering
\begin{subfigure}[b]{0.30\textwidth}
\includegraphics[width=\textwidth]{ 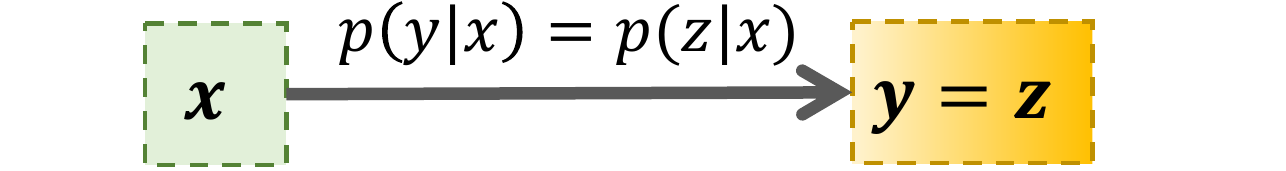}
\caption{}
\label{fig.mathematical foundation1}
\end{subfigure}
\hfill
\begin{subfigure}[b]{0.30\textwidth}
\includegraphics[width=\textwidth]{ 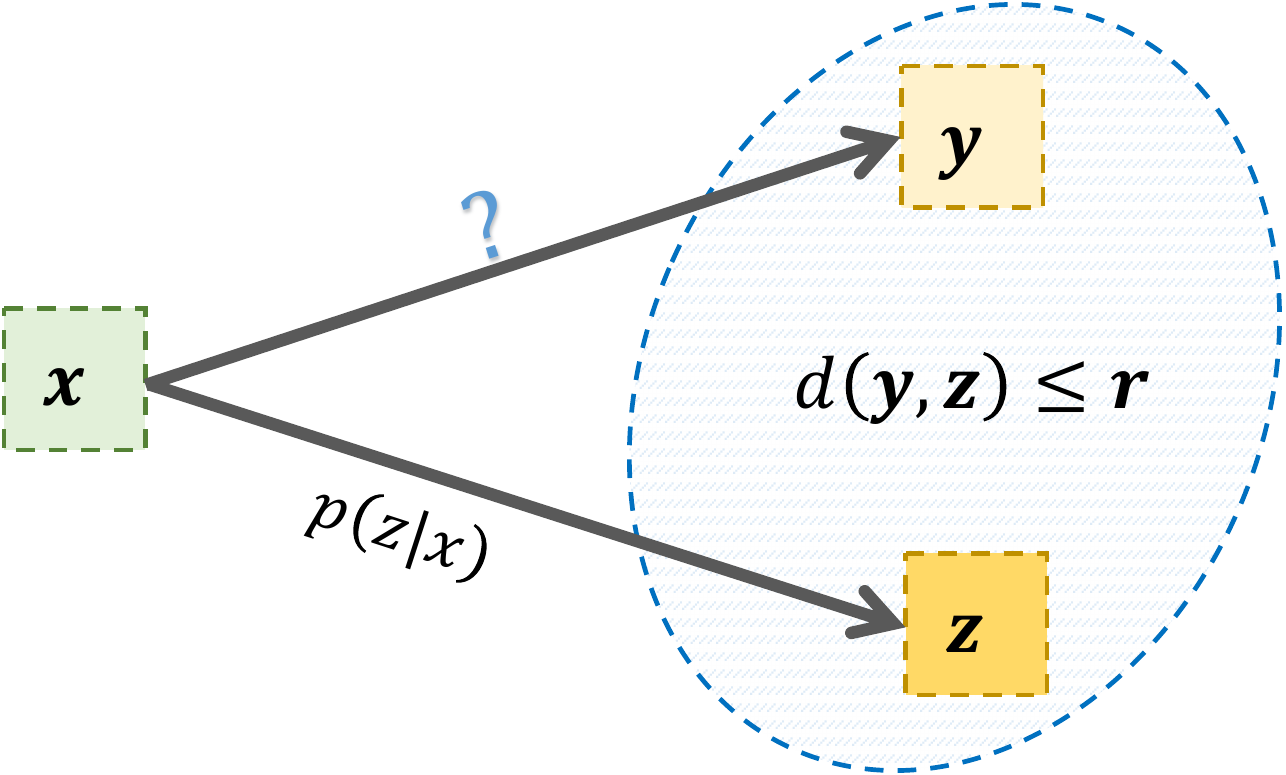}
\caption{}
\label{fig.mathematical foundation2}
\end{subfigure}
\caption{(a) $\mathbf{y}$ and $\mathbf{z}$ are identical ($\mathbf{y}=\mathbf{z}$). (b) The distance between $\mathbf{y}$ and $\mathbf{z}$ is bounded ($d(\mathbf{y},\mathbf{z}) \leq \mathbf{r}$).} 
\label{fig.mathematical foundation}
\end{figure}

Given three random variables \(\mathbf{x}\), \(\mathbf{y}\), and \(\mathbf{z}\), where the information of \(p(z|x)\) is known and certain constraints exist between \(\mathbf{y}\) and \(\mathbf{z}\), our objective is to compute
\(p(y|x)\) by leveraging these constraints. 
In the case that \(\mathbf{y}\) and \(\mathbf{z}\) are equal, i.e., \(\mathbf{y} = \mathbf{z}\), we have \(p(y|x) = p(z|x)\), as shown in Fig.~\ref{fig.mathematical foundation1}. Conversely, if $\mathbf{y}$ is not equal to $\mathbf{z}$ and their difference is bounded by an inequality function, we can define a convolutional version of  conditional probability, as shown in Fig. \ref{fig.mathematical foundation2}.
\begin{definition}[Convolutional Conditional Probability]\label{prop.convolutional conditional probability}
Given $p(z|x)$, if $\mathbf{y}$ and $\mathbf{z}$ are conditionally independent given $\mathbf{x}$, then convolutional conditional probability is defined as 
\emph{ $$p_{c}(y|x):= p(y|x, d(\mathbf{y}, \mathbf{z}) \leq \mathbf{r}).$$}
\end{definition}
Here, \(d(\mathbf{y}, \mathbf{z}) \leq \mathbf{r}\) is the inequality condition, where \(\mathbf{r}\) is a threshold random variable with known cumulative distribution function \(F(r)\), and \(d(\cdot, \cdot)\) is the distance metric for two random variables. 
The calculation of convolutional conditional probability $p_c(y|x)$ is summarized in the subsequent proposition:
\begin{proposition}
The convolutional conditional probability satisfies
\emph{
\begin{equation}\label{eq.conditional probability}
\begin{aligned}
p_{c}(y|x) &=
\frac{
\int_{z} \left(1-F(d(y, z)) \right) p(z|x) \, \mathrm{d}z}{\int_{y} \int_{z} \left(1-F(d(y, z)) \right) p(z|x) \, \mathrm{d}z \, \mathrm{d}y}
\\
&\propto \int_{z} \left(1-F(d(y, z)) \right) p(z|x) \, \mathrm{d}z.    
\end{aligned}
\end{equation}
}
\end{proposition}
\begin{proof}
According to Bayes' law, we have
\begin{equation}\label{eq.proof Bayes}
\begin{aligned}
p_c(y|x) 
= \frac{P(d(\mathbf{y}, \mathbf{z}) \leq \mathbf{r}|x, y) p_\text{pri}(y| x)}{\int_{y} P(d(\mathbf{y}, \mathbf{z}) \leq \mathbf{r}|x, y) p_\text{pri}(y| x) \, \mathrm{d}y}.
\end{aligned}
\end{equation} 
In \eqref{eq.proof Bayes}, $p_\text{pri}(y|x)$ is the prior of $p(y|x)$. It is chosen as an uninformative probability since we have no knowledge of $p(y|x)$ \cite{friston2002classical}:
\begin{equation}\label{eq.proof flat prior}
p_\text{pri}(y|x) = C, \quad C>0.  
\end{equation}
The likelihood term $P(d(\mathbf{y}, \mathbf{z}) \leq \mathbf{r}|x, y)$ in \eqref{eq.proof Bayes} is simplified as
\begin{equation}\label{eq.proof likelihood}
\begin{aligned}
&P(d(\mathbf{y}, \mathbf{z}) \leq \mathbf{r}|x, y) 
\\
=&P(d(y, \mathbf{z}) \leq \mathbf{r}|x, y) 
\\
=& \int_{z} \left(1-F(d(y, z)) \right) p(z|x,y) \, \mathrm{d}z
\\
=& \int_{z} \left(1-F(d(y, z)) \right) p(z|x) \, \mathrm{d}z.
\end{aligned}
\end{equation}
Substituting \eqref{eq.proof flat prior} and \eqref{eq.proof likelihood} into \eqref{eq.proof Bayes}, we have \eqref{eq.conditional probability}. 
Note that the final equation of \eqref{eq.proof likelihood} holds due to the conditional independence of $\mathbf{y}$ and $\mathbf{z}$, given $\mathbf{x}$.
\end{proof}

\begin{remark}
The calculation of convolutional conditional probability \emph{$p_c(y|x)$} resembles convolution operation. In the convolution operation, a kernel is applied over an input space to generate a modified output. Here, \( 1-F(d(y, z)) \) serves as the kernel function, which is a weighting coefficient of  \( p(z|x) \) based on the distance between $\mathbf{y}$ and  $\mathbf{z}$ .  
\end{remark}

We want to emphasize that $p(y|x)$, as used in \eqref{eq.total probability rule} and \eqref{eq.Bayes' law}, can become any form of conditional probabilities. Actually, $p_c(y|x)$ represents a specific form of conditional probability, which measures the probability of the event $\{\mathbf{y}=y\}$, conditioned on two events $\{\mathbf{x} = x\}$ and $\{d(\mathbf{y}, \mathbf{z}) \leq \mathbf{r}\}$. Compared to standard definition, convolutonal conditional probability has the third event which describes the upper bound between two random variables.  Under this new definition, one can substitute $p(y|x)$ with $p_c(y|x)$ in \eqref{eq.total probability rule} and \eqref{eq.Bayes' law} to construct two new rules:
\begin{subequations}\label{eq.convolutional rule}
\begin{align}
p(y) &= \int_x p_c(y|x) p(x) \, \mathrm{d}x, \label{eq.convolutional total probability rule} \\
p(x|y) &= \frac{p(x) p_c(y|x)}{\int_x p(x) p_c(y|x) \, \mathrm{d}x}. \label{eq.convolutional Bayes' law}
\end{align}
\end{subequations}
Here, \eqref{eq.convolutional total probability rule} can be regarded as a generalized total probability rule and \eqref{eq.convolutional Bayes' law} can be regarded as a generalized Bayes' law. Note that
$p(y)$ and $p(x|y)$ in \eqref{eq.convolutional rule} are distinct from those notations in \eqref{eq.rule} because they implicitly condition on the third 
event $\{d(\mathbf{y}, \mathbf{z}) \leq \mathbf{r}\}$, while this event in \eqref{eq.rule} is reduced to $\{\mathbf{y}=\mathbf{z}\}$. That is to say, an equivalence event is  implicitly conditioned in standard definition \eqref{eq.rule}. It can be proven that total probability rule \eqref{eq.total probability rule} and Bayes' law \eqref{eq.Bayes' law} are the special cases of \eqref{eq.convolutional total probability rule} and \eqref{eq.convolutional Bayes' law}, as described in Proposition \ref{prop.special case}.

\begin{proposition}[Limiting Property]\label{prop.special case}
Choose $\mathbf{r}$ as the exponential distribution, i.e., $\mathbf{r}\sim \mathrm{Exp}(\lambda)$ and $d(y,z) = \|y-z \|^2/ \sigma^2$, where $\lambda >0$ is  exponential parameter, and $\sigma >0$ is scale parameter. As $\sigma \to 0$, $p_c(y|x)$ reduces to $p(y|x)$, \eqref{eq.convolutional total probability rule} reduces to total probability rule \eqref{eq.total probability rule}, and  \eqref{eq.convolutional Bayes' law} reduces to Bayes' law \eqref{eq.Bayes' law}.
\end{proposition}
\begin{proof}
To simplify the derivation, let us define 
\begin{equation}\nonumber
\begin{aligned}
\eta := -\frac{\lambda \|y-z \|^2}{ \sigma^2}, \,
\gamma := \frac{\lambda^{\frac{n}{2}}}{\left(\pi \sigma^2 \right)^{\frac{n}{2}}},
\end{aligned}
\end{equation}
where $n$ is the dimension of $y$ or $z$.
From Proposition \ref{prop.convolutional conditional probability},
we have
\begin{equation}\nonumber
\begin{aligned}
p_c(y|x) &=
\frac{
\int_{z} e^\eta p(z|x) \, \mathrm{d}z}{\int_{y} \int_{z} e^{\eta} p(z|x) \, \mathrm{d}z \, \mathrm{d}y} 
\\
&= \frac{\gamma
\int_{z} e^{\eta} p(z|x) \, \mathrm{d}z}{\int_{y} \gamma 
 \int_{z} e^\eta p(z|x) \, \mathrm{d}z \, \mathrm{d}y}. 
\end{aligned}
\end{equation}
As $\sigma \to 0$, $\gamma
e^{\eta}$ becomes the Dirac function of $z$ centered at $y$. Thus, we have
\begin{equation}\label{eq.convergence}
\begin{aligned}
\lim_{\sigma \to 0} p_c(y|x) =& 
\lim_{\sigma \to 0} \frac{
\int_{z}  p(z|x) \delta(z-y) \, \mathrm{d}z}{\int_{y} \int_{z} p(z|x) \delta(z-y) \, \mathrm{d}z \, \mathrm{d}y}
\\
=& 
p(y|x).    
\end{aligned}  
\end{equation}
As a result, \( p_c(y|x) \) converges to \( p(y|x) \) as \( \sigma \) approaches 0, and accordingly, \eqref{eq.convolutional total probability rule} and \eqref{eq.convolutional Bayes' law} converge to ~\eqref{eq.total probability rule} and \eqref{eq.Bayes' law}, respectively.
\end{proof}

Proposition~\ref{prop.special case} elucidates that the kernel function converges to a Dirac delta function at \( y \) as the scale parameter  tends to zero. As a result, the convolutional conditional probability reduces to its standard version.
For finite values of scale parameter, this new definition considers a controllable amount of uncertainty governed by the scale parameter, offering an extension to the previous one.

\section{Convolutional Bayesian Filtering}\label{sec.convolutional Bayesian filtering}
In this section, we demonstrate how model mismatches in the filtering problem can be explicitly addressed using convolutional conditional probability. This is achieved by conditioning on an additional event representing the error bound between the system model and the real system. Further, by substituting the total probability rule and Bayes' law with \eqref{eq.convolutional total probability rule} and \eqref{eq.convolutional Bayes' law} in the Bayes filter, we can establish a more generalized filtering framework called convolutional Bayesian filtering.
\subsection{Uncertain Hidden Markov Model with Model Mismatch}
The essence of Bayesian filtering is to find a balance
between the stoasticities of  state transition and state observation. The stochasticity of the former comes from inherent randomness in the environment dynamics while that of the latter comes from sensor noises.  These two stochastic processes are typically represented by hidden Markov model (HMM):
\begin{equation} \label{eq.hmm}
\begin{aligned}
\mathbf{x}_0 &\sim p_0(x_0),\\
\mathbf{x}_t &\sim p(x_t|x_{t-1}),\\
\mathbf{y}_t &\sim p(y_t|x_t).
\end{aligned}
\end{equation}
Here, $\mathbf{x}_t \in \mathbb{R}^n$ 
is the system state and $\mathbf{y}_t \in \mathbb{R}^m$ is the corresponding measurement. 
Besides, $p_0$ denotes the probability distribution of the initial state $\mathbf{x}_0$, $p(x_t|x_{t-1})$ represents the transition probability, and $p(y_t|x_t)$ is the output probability. 

The standard HMM implicitly assumes that the real system is perfectly modelled, i.e., \eqref{eq.hmm} is an exact description of system dynamics and measurement sensors. However, perfect information about transition or output probabilities is often unattainable due to parametric variation, unmodeled dynamics or external disturbances. In other words, there must be some model errors in engineering practice. This error can lead to significant accuracy degradation in state estimation if not properly considered. To build an HMM with model mismatch, we have to distinguish two kinds of states: the real state and the virtual state. The former is an accurate yet unattainable description of the system. The latter is an artificial construct generated by nominal models and does not exist in the physical world. 
We use $\bar{\mathbf{x}}_t$ to represent the virtual state and  ${\mathbf{x}}_t$ to represent the real state.  
Likewise, the real measurement is denoted as $\mathbf{y}_t$, and its virtual counterpart, which is generated by the nominal output probability, is denoted as $\bar{\mathbf{y}}_t$. 
The HMM with model mismatch is depicted in Fig.~\ref{fig.HMM} and outlined in \eqref{eq.modified hmm} as follows:
\begin{subequations} \label{eq.modified hmm}
\begin{align}
&\mathbf{x}_0 \sim p_0(x_0),\nonumber
\\
&\bar{\mathbf{x}}_t \sim p(\bar{x}_t|x_{t-1}),\nonumber
\\
&\bar{\mathbf{y}}_t \sim p(\bar{y}_t|x_t),\label{eq.stochasticity}
\\
&d(\mathbf{x}_t, \bar{\mathbf{x}}_t) \leq \mathbf{r}_x ,\nonumber
\\ &d(\mathbf{y}_t, \bar{\mathbf{y}}_t) \leq \mathbf{r}_y. \label{eq.mismatch}
\end{align}
\end{subequations}
Here, $p(\bar{x}_t|{x}_{t-1})$ and $p(\bar{y}_t|{x}_{t})$ denote the nominal transition probability and nominal output probability respectively; $\mathbf{r}_x$ and $\mathbf{r}_y$ are the threshold random variables depicting the upper bound of model mismatch, with their distributions characterized by the cumulative distribution functions $F_x$ and $F_y$, respectively. Importantly, $\mathbf{r}_x$ is assumed to be independent of both $\mathbf{x}_t$ and $\bar{\mathbf{x}}_t$, and a similar independence assumption is made for $\mathbf{r}_y$ relative to $\mathbf{y}_t$ and $\bar{\mathbf{y}}_t$. This new form in \eqref{eq.modified hmm} is called uncertain hidden Markov model.

It is crucial to differentiate between the concepts of system stochasticity (see \eqref{eq.stochasticity}) and model mismatch (see \eqref{eq.mismatch}). The distinction hinges on the presence of known mathematical forms. System stochasticity can be accurately modeled using explicit distributions with associated parameters. In contrast, model mismatch refers to the inherent limitations and uncertainties in a model’s ability to represent the real system. Typically, this can only be quantified by an upper bound that reflects the extent to which the model deviates from reality. If we can only acquire the bound of system stochasticity, it inherently becomes a special case of model mismatch. Conversely, if the distribution of model mismatch is determined, it then becomes part of the system’s stochasticity. This distinction is pivotal in understanding and building uncertain HMM.

\begin{figure}
    \centering
    \includegraphics[width=1\linewidth]{ 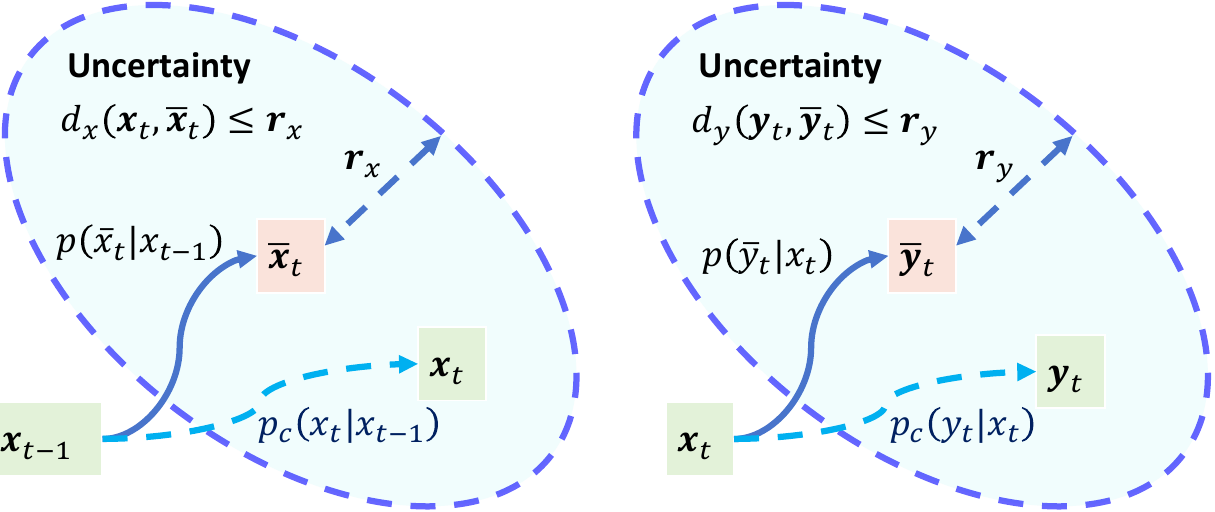}
    \caption{Illustration of uncertain HMM with model mismatch. The nominal transition probability in HMM projects the real state in the previous time, denoted as $\mathbf{x}_{t-1}$, to a model-predicted virtual state $\bar{\mathbf{x}}_t$. However, due to model mismatch, $\mathbf{x}_{t-1}$ transitions to the real state $x_t$ in the physical world. The transition from $\bar{\mathbf{x}}_t$ to $\mathbf{x}_t$ is inaccessible; instead, their distance $d(\mathbf{x}_t, \bar{\mathbf{x}}_t)$ are bounded by a
    threshold random variable $\mathbf{r}_x$, i.e., $d(\mathbf{x}_t, \bar{\mathbf{x}}_t) \leq \mathbf{r}_x$. Analogously,  $d(\mathbf{y}_t, \bar{\mathbf{y}}_t) \leq \mathbf{r}_y$ reflects the measurement model mismatch.}
    \label{fig.HMM}
\end{figure}

\begin{remark}
When the real system is perfectly modelled, i.e., $\mathbf{x}_t = \bar{\mathbf{x}}_t$, $\mathbf{y}_t = \bar{\mathbf{y}}_t$, we have $p(\bar{x}_t|x_{t-1}) = p(x_t|x_{t-1})$ and $p(\bar{y}_t|x_t) = p(y_t|x_t)$. In this case, the uncertain HMM \eqref{eq.modified hmm} reduces to the standard HMM  \eqref{eq.hmm}.
\end{remark}

\begin{remark}
The two nominal probabilities $p(\bar{x}_t|x_{t-1})$ and $p(\bar{y}_t|x_t)$ can be written to state space model (SSM):
\begin{equation}
\begin{aligned}
\bar{\mathbf{x}}_{t+1} = \bar{f}(\mathbf{x}_t, \bar{\mathbf{\xi}}_t),\\
\bar{\mathbf{y}}_t = \bar{g}(\mathbf{x}_t, \bar{\mathbf{\zeta}}_t).     
\end{aligned}
\end{equation}
Here, $\bar{f}$ is nominal transition model, $\bar{g}$ is nominal measurement model, $\bar{\xi}_t$ is virtual process noise, and $\bar{\zeta}_t$ is virtual measurement noise.
\end{remark}
In essence, HHM and SSM are just different representations of the same system model, and they can be converted into each other \cite{li2023reinforcement, cao2023generalized}. For example, consider the HMM's nominal transition probability expressed as \( p(\bar{x}_t|x_{t-1}) = \mathcal{N}(\bar{x}_t; Ax_{t-1}, Q) \). This can be equivalently represented in the SSM format as \( \bar{\mathbf{x}}_{t} = A\mathbf{x}_{t-1} + \bar{\mathbf{\xi}}_{t-1} \), where \( \bar{\mathbf{\xi}}_{t-1} \sim \mathcal{N}(\bar{\mathbf{\xi}}_{t-1}; 0, Q) \) with \( Q \) denoting the covariance matrix of the virtual process noise. Similarly, an HMM with a nominal transition probability defined by a Laplace distribution, \( p(\bar{x}_t|x_{t-1}) = \mathrm{Laplace}(\bar{x}_t; Ax_{t-1}, Q) \), can be represented in SSM format as \( \bar{\mathbf{x}}_{t} = A\mathbf{x}_{t-1} + \bar{\mathbf{\xi}}_{t-1} \), with \( \bar{\mathbf{\xi}}_{t-1} \) following the Laplace distribution \( \mathrm{Laplace}(\bar{\mathbf{\xi}}_{t-1}; 0, Q) \).

\subsection{Filtering Algorithm} 
When the system is perfectly modelled as in \eqref{eq.hmm}, Bayesian filtering serves as an ideal framework to calculate the posterior of system state by iteratively performing \eqref{eq.BF prediction} and  \eqref{eq.BF update}:
\begin{subequations}
\begin{align}
p(x_t|y_{1:t-1}) &= \int p(x_t|x_{t-1})p(x_{t-1}|y_{1:t-1}) \,\mathrm{d}x_{t-1}, \label{eq.BF prediction}
\\
p(x_t|y_{1:t}) &= \frac{p(x_t|y_{1:t-1})p(y_t|x_t)}{\int p(x_t|y_{1:t-1})p(y_t|x_t) \,\mathrm{d}x_{t}}. \label{eq.BF update}
\end{align}    
\end{subequations}
Here, $p(x_t|y_{1:t})$ is the posterior, $p(x_t|y_{1:t-1})$ is the prior.
As a tradition, \eqref{eq.BF prediction} is called prediction and \eqref{eq.BF update} is called update. These two equations originate from total probability rule and Bayes' law respectively. In fact, almost all existing Bayesian filtering algorithms adhere to this framework, with their differences lying in how these two equations are calculated.

When considering the HMM with model mismatch \eqref{eq.modified hmm}, we need to shift the mathematical foundation to convolutional rules, as demonstrated in \eqref{eq.convolutional rule}. First, let us redefine some core probability distributions in Bayesian filtering, including posterior distribution, prior distribution, transition probability, and output probability.  The redefinition relies on the uncertain HMM:
\begin{equation}\label{eq.definitions}
\begin{aligned}
p_c(x_t|y_{1:t}):&=
p(x_t | y_{1:t}, d_x(\mathbf{x}_i, \bar{\mathbf{x}}_i) \leq \mathbf{r}_x,  d_y(\mathbf{y}_i, \bar{\mathbf{y}}_i) \leq \mathbf{r}_y), 
\\
p_c(x_t|y_{1:t-1}):&=p(x_t | y_{1:t-1}, d_x(\mathbf{x}_i, \bar{\mathbf{x}}_i) \leq \mathbf{r}_x, d_y(\mathbf{y}_j, \bar{\mathbf{y}}_j) \leq \mathbf{r}_y),
\\
p_c(x_t|x_{t-1}):&= p(x_t|x_{t-1}, d_x(\mathbf{x}_t, \bar{\mathbf{x}}_{t}) \leq \mathbf{r}_x),
\\
p_c(y_t|x_{t}) :&= p(y_t|x_{t}, d_y(\mathbf{y}_t, \bar{\mathbf{y}}_{t}) \leq \mathbf{r}_y),
\\
i&= 1,2,...,t,\; j=1,2,...,t-1.
\end{aligned}  
\end{equation}
Here,  \( p_c(x_t|y_{1:t}) \), \( p_c(x_t|y_{1:t-1}) \),  \( p_c(x_t|x_{t-1}) \) and  $p_c(y_t|x_t)$ are convolutional distributions, each corresponding to their respective physical meanings. 
Then, we will illustrate how to use these definitions in \eqref{eq.definitions} to derive a new Bayesian filtering framework to handle model mismatch. We begin with the assumption of conditional independence:

\begin{assumption}[Conditional Independence]\label{assump.conditional independence x}
$\mathbf{x}_t$ and $\bar{\mathbf{x}}_t$ are conditionally independent given $\mathbf{x}_{t-1}$, i.e.,
\begin{equation}\nonumber
p(\bar{x}_t|x_t, x_{t-1}) = p(\bar{x}_t|x_{t-1}).  
\end{equation}
Besides, $\mathbf{y}_t$ and $\bar{\mathbf{y}}_t$ are conditionally independent given $\mathbf{x}_{t}$.
\end{assumption}

This assumption suggests that the virtual state can be inferred directly from the past state, without additional information from the current state. Also, the virtual measurement can be inferred directly from the state, without additional information from the real measurement. This assumption originates from the philosophical belief that the physical world and the modeling of a system are mutually exclusive at any given moment; that is, the act of modeling does not affect the system in the physical world, nor does the physical system influence its nominal model. This principle is crucial for estimating the transition and output probabilities using their nominal models. 
Under the assumption of conditionally independence, we can obtain the main result of the paper:

\begin{theorem}[Convolutional Bayesian filtering]
Under Assumption~\ref{assump.conditional independence x}, the convolutional Bayesian filtering is calculated recursively by \eqref{eq.convolutional Bayesian filtering prediction} and \eqref{eq.convolutional Bayesian filtering update}:
\begin{subequations}\label{eq.convolutional Bayesian filtering}
\begin{align}
p_c(x_t|y_{1:t-1}) 
&= \int p_c(x_t|x_{t-1})p_c(x_{t-1}|y_{1:t-1}) \,\mathrm{d}x_{t-1}, \label{eq.convolutional Bayesian filtering prediction}
\\
p_c(x_t|y_{1:t}) 
&= \frac{p_c(x_t|y_{1:t-1})p_c(y_t|x_t)}{\int p_c(x_t|y_{1:t-1})p_c(y_t|x_t) \,\mathrm{d}x_{t}}. \label{eq.convolutional Bayesian filtering update}   
\end{align}    
\end{subequations}
Here, the convolutional transition probability is
\begin{subequations}\label{eq.convolutional hmm probability}
\begin{align}
\label{eq.convolutional transition probability}
p_c(x_t|x_{t-1}) &= \frac{\int_{\bar{x}_t} \left(1-F_{x}(d_x(x_t, \bar{x}_t)) \right) p(\bar{x}_t|x_{t-1}) \, \mathrm{d} \bar{x}_t}{\int_{x_t} \int_{\bar{x}_t} \left(1-F_{x}(d_x(x_t, \bar{x}_t)) \right) p(\bar{x}_t|x_{t-1}) \, \mathrm{d} \bar{x}_t \, \mathrm{d} x_t},
\\
\intertext{and the convolutional output probability is}
\label{eq.convolutional output probability}
p_c(y_t|x_{t}) &= \frac{\int_{\bar{y}_t} \left(1-F_{y}(d_y(y_t, \bar{y}_t)) \right) p(\bar{y}_t|x_{t}) \, \mathrm{d} \bar{y}_t}{\int_{y_t} \int_{\bar{y}_t} \left(1-F_{y}(d_y(y_t, \bar{y}_t)) \right) p(\bar{y}_t|x_{t}) \, \mathrm{d} \bar{y}_t \, \mathrm{d}{y}_t}.
\end{align}
\end{subequations} 
\end{theorem}
\begin{proof}
The transition and output probabilities in \eqref{eq.hmm} are standard conditional probabilities. By leveraging the definition of convolutional conditional probability, we can derive the convolutional forms of \eqref{eq.hmm} as in \eqref{eq.convolutional hmm probability}.
Consequently, their convolutional counterparts are referred to as the convolutional transition probability \eqref{eq.convolutional transition probability} and the convolutional output probability \eqref{eq.convolutional output probability}, respectively. Then by utilizing \eqref{eq.convolutional rule}, we can deduce convolutional Bayesian filtering \eqref{eq.convolutional Bayesian filtering}.
\end{proof}

\begin{remark}
In the iterative process, \eqref{eq.convolutional Bayesian filtering prediction} and \eqref{eq.convolutional Bayesian filtering update} resemble the prediction and update steps of Bayesian filtering, respectively. The only difference is that all the probability distributions are transformed into their convolutional counterparts. Therefore, we refer to this iterative process as convolutional Bayesian filtering.
\end{remark}

\subsection{Analytical Form in Gaussian Cases}
A major challenge in convolutional Bayesian filtering is the difficulty in computing the integrals in  \eqref{eq.convolutional transition probability} and \eqref{eq.convolutional output probability}, as their analytical solutions generally do not exist.
However, an exceptional case arises when distance metrics are represented as quadratic forms, threshold distributions are chosen as exponential distributions, and virtual noises are characterized as additive Gaussian. In this specific case, it is possible to derive an analytical version of convolutional Bayesian filtering.

\begin{corollary}\label{corollary.analytical}
Consider the following nominal system model
\begin{equation}\label{eq.nonlinear Guassian}
\begin{aligned}
p(\bar{x}_t|x_{t-1}) &= \mathcal{N}(\bar{x}_t;f(x_{t-1}),Q),
\\
p(\bar{y}_t|x_t) &= \mathcal{N}(\bar{y}_t;g(x_t),R).
\end{aligned}    
\end{equation}
If $d_x(\mathbf{x}, \bar{\mathbf{x}}) = \|\mathbf{x} - \bar{\mathbf{x}} \|^2$,  $d_y(\mathbf{y}, \bar{\mathbf{y}}) = \|\mathbf{y} - \bar{\mathbf{y}} \|^2$, $\mathbf{r}_x \sim \mathrm{Exp}(\alpha)$ and $\mathbf{r}_y \sim \mathrm{Exp}(\beta)$ with $\alpha, \beta>0$ being exponential coefficients, we have
\begin{subequations}
\begin{align}
p_c(x_t|x_{t-1}) &= \mathcal{N}(x_t; f(x_{t-1}), Q + 1/(2\alpha) \cdot I_{n \times n}), \label{eq.analytical x}
\\
p_c(y_t|x_{t}) &= \mathcal{N}(y_t; g(x_t), R + 1/(2\beta) \cdot I_{m \times m)}.\label{eq.analytical y}
\end{align}
\end{subequations}
\end{corollary}

\begin{proof}
The proof is provided only for the first part, namely, proving the analytical form of $p_c(x_t|x_{t-1})$ in \eqref{eq.analytical x}. The second part \eqref{eq.analytical y}  can be proved in a similar manner. According to \eqref{eq.convolutional transition probability}, we have
\begin{equation}\nonumber
\begin{aligned}
& p_c(x_t|x_{t-1})
\\
\propto & \int_{\bar{x}_t} \left(1-F_{x}(d_x(x_t, \bar{x}_t)) \right) p(\bar{x}_t|x_{t-1}) \, \mathrm{d} \bar{x}_t, \\
= & \int_{\bar{x}_t} e^{-\alpha \|x_t - \bar{x}_t\|^2} e^{-\frac{1}{2}\left\|\bar{x}_t - f(x_{t-1})\right\|^2_{Q^{-1}} }
\, \mathrm{d} \bar{x}_t.
\end{aligned}
\end{equation}
By completing the square, we have
\begin{equation}\label{eq.completing square}
\begin{aligned}
&\alpha \|x_t - \bar{x}_t\|^2 + \frac{1}{2}\left(\bar{x}_t - f(x_{t-1})\right)^{\top}Q^{-1}\left(\bar{x}_t - f(x_{t-1})\right) \\
=& \frac{1}{2} \Big(
\bar{x}_t^{\top}(2\alpha I_{n \times n} + Q^{-1})\bar{x}_t \\
&-2\bar{x}_t^{\top}(2\alpha I_{n \times n}x_t + Q^{-1}f(x_{t-1})) + C
\Big) \\
=& \frac{1}{2} \Big(
\|\bar{x}_t - (2\alpha I_{n \times n} 
\\
&+ Q^{-1})^{-1}(2\alpha x_t + Q^{-1}f(x_{t-1}))\|^2_{2\alpha I_{n \times n} + Q^{-1}} + C
\Big).
\end{aligned}
\end{equation}
where \( C \) indicates terms that do not depend on \( \bar{x}_t \).
The integral of  \eqref{eq.completing square} over \( \bar{x}_t \) is proportional to
\begin{equation}\nonumber
e^{-\frac{1}{2}(x_t - f(x_{t-1}))^\top (1/(2\alpha) \cdot I_{n \times n} + Q)^{-1} (x_t - f(x_{t-1}))},
\end{equation}
where \( 1/(2\alpha) \cdot I_{n \times n} + Q \) is the covariance matrix of the convolutional transition probability and \( f(x_{t-1}) \) is its mean. This results in an analytical form of the convolutional transition probability:
\begin{equation}\nonumber
p_c(x_t|x_{t-1}) = \mathcal{N}(x_t; f(x_{t-1}), Q + 1/(2\alpha) \cdot I_{n \times n}).
\end{equation}
\end{proof}

This corollary shows that by using quadratic distance metrics and choosing exponential threshold variables, the covariance matrix of the convolutional transition probability for system \eqref{eq.nonlinear Guassian} essentially becomes the nominal covariance matrix plus a constant matrix related to the exponential coefficient. As the exponential coefficient
increases, the exponential distribution becomes more concentrated, with its mean and variance tending towards zero. This implies that the uncertain HMM becomes increasingly deterministic. When the exponential coefficient becomes infinity, the effect of model mismatch diminishes, and the convolutional transition probability $p_c(x_t|x_{t-1})$ reduces to the nominal transition  probability $p(x_t|x_{t-1})$. This analysis is equally applicable to 
convolutional output probability and nominal output probability.

For linear Gaussian case, where the system in \eqref{eq.nonlinear Guassian} satisfies $f(x_{t-1}) = Ax_{t-1}$ and $g(x_t) = Cx_t$, standard Bayesian filtering simplifies to the canonical Kalman filter. By using Corollary~\ref{corollary.analytical}, the canonical Kalman filter can be transformed into its convolutional version by only replacing the covariance matrix of process noise $Q$ with $Q + 1/{(2\alpha)}$, and the covariance matrix of measurement noise $R$ with $R + 1/{(2\beta)}$. 

\begin{remark}
The resulting method is an outlier-robust Kalman filter (KF), which we name as convolutional KF (ConvKF). Unlike the robust regression KF that employs Huber loss \cite{huber2004robust} or correntropy loss \cite{chen2017maximum,tao2023outlier,tao2023maximum}, and the student-t KF \cite{roth2013student, agamennoni2012approximate} designed for handling non-Gaussian heavy-tailed distributions, ConKF offers several benefits: First, it quantitatively considers the impact of model mismatch with a clear probabilistic meaning; second, it preserves the original structure of  KF, maintaining the conjugate nature of Gaussian distributions without increasing the computational burden; third, our method is in alignment with the well-established results for engineering practice of the KF, as discussed in Chapter 6.1 of \cite{anderson2012optimal} and Theorem 7.6 of \cite{jazwinski2007stochastic}: if the modeling of noise covariance is imprecise, it is common practice to opt for a larger covariance in the KF. This treatment is proven to preserve stability, albeit resulting in a more conservative filter.
\end{remark}

\section{Approximation via Exponential Density Rescaling}\label{sec.theoretical approximation}
Except for the Gaussian case addressed in Corollary \ref{corollary.analytical}, the convolutional conditional probabilities in Bayesian filtering typically lack analytical forms. In this section, we introduce an approximation technique for computing the convolutional conditional probability, namely the exponential density rescaling technique. Moreover, we offer a theoretical explanation for this technique using the theory of information bottleneck.

\subsection{Exponential Density Rescaling}
When the
distance metric is defined
in terms of relative entropy, the transition probability and
output probability can be approximated by simply reformulating them
into exponential forms with fractional powers. Specifically, we have the following theorem:
\begin{theorem}[Exponential Approximation]\label{theorem3.power approximation}
When the distance metrics $d_x$ and $d_y$ are chosen as
\begin{equation}\nonumber
d_x(\mathbf{x}_t, \bar{\mathbf{x}}_t) = D_{\mathrm{KL}}(\hat{p}_{\mathbf{x}_t}|| \hat{p}_{\bar{\mathbf{x}}_t}),\; d_y(\mathbf{y}_t, \bar{\mathbf{y}}_t) = D_{\mathrm{KL}}(\hat{p}_{\mathbf{y}_t}|| \hat{p}_{\bar{\mathbf{y}}_t}),
\end{equation} 
with \(\hat{p}_{\mathbf{x}_t}(x) = \delta(x-\mathbf{x}_t)\) and \(\hat{p}_{\mathbf{y}_t}(y) = \delta(y-\mathbf{y}_t)\) representing the empirical distribution,
and \(\mathbf{r}_x \sim \mathrm{Exp}(\alpha)\), \(\mathbf{r}_y \sim \mathrm{Exp}(\beta)\) with $\alpha, \beta >0$, the convolutional transition probability and convolutional output probability can be approximated as
\begin{equation}\nonumber
\begin{aligned}
p_c(x_t|x_{t-1}) &\approx \frac{p(x_t|x_{t-1})^{\frac{\alpha}{\alpha+1}}}{\int p(x_t|x_{t-1})^{\frac{\alpha}{\alpha+1}} \, \mathrm{d} x_t} := p_e (x_t|x_{t-1}),
\\
p_c(y_t|x_{t}) &\approx \frac{p(y_t|x_{t})^{\frac{\beta}{\beta+1}}}{\int p(y_t|x_{t})^{\frac{\beta}{\beta+1}} \, \mathrm{d} y_t} \,\,\,\,\,\,\,:= p_e (y_t|x_t).
\end{aligned}  
\end{equation}
\end{theorem}
\begin{proof}
We prove the theorem only for the convolutional output probability, because the proof logic for the convolutional transition probability is analogous and thus omitted. Besides, the proof is confined to cases where the sample space is finite.
Supposing the sample space has $k$ elements, we define the set of all the probability distributions as $$\Delta_k = \left\{ p \in \mathbb{R}^k: p = [p_1, p_2, \ldots, p_k], \sum_i p_i = 1, p_i>0, \forall i \right\}.$$ 
By Proposition \ref{prop.convolutional conditional probability}, we have
\begin{equation}\nonumber
p_c(y_t|x_t) \propto P(d_{y}(\mathbf{y}_t, \bar{\mathbf{y}}_t)\leq \mathbf{r}_y|y_t, x_{t}).     
\end{equation}
According to \eqref{eq.proof likelihood}, we have
\begin{equation}\nonumber
\begin{aligned}
&P(d_{y}(\mathbf{y}_t, \bar{\mathbf{y}}_t) \leq \mathbf{r}_y|y_t, x_{t}) 
\\
=& \int_{\bar{y}_t} \left(1-F_{y}(d_{y}(y_t, \bar{y}_t)) \right) p(\bar{y}_t|x_{t}) \, \mathrm{d} \bar{y}_t
\\
=& \int_{\bar{y}_t} e^{-\beta(D_{\text{KL}}(\hat{p}_{y_t} || \hat{p}_{\bar{{y}}_t})) } p(\bar{y}_t|x_{t}) \, \mathrm{d} \bar{y}_t
\\
=& \mathbb{E}_{\bar{\mathbf{y}}_t \sim p(\bar{y}_t|x_t)}\left\{ e^{-\beta(D_{\text{KL}}(\hat{p}_{y_t} || \hat{p}_{\bar{\mathbf{y}}_t})) } \right\}.
\end{aligned}
\end{equation}
For any $o\in \mathbb{R}^k$, we define $o' := (o_1, o_2, \ldots, o_{k-1}) \in \mathbb{R}^{k-1}$ and  $C(o) \in \mathbb{R}^{k \times k}$ satisfying $[C(o)]_{(i,j)} = o_i \mathds{1}_{(i=j)} - o_i o_j$. Under the assumption of finite sample space, we denote \(\mathbf{\hat{q}} = \hat{p}_{\bar{\mathbf{y}}_t} \in \Delta_{k}\) and $p = \hat{p}_{{y}_t} \in \Delta_{k}$. Thus, we have
\begin{equation}\nonumber
\begin{aligned}
\mathbb{E} \left\{ e^{-\beta(D_{\text{KL}}(\hat{p}_{y_t} || \hat{p}_{\bar{\mathbf{y}}_t}))} \right\}
=  \mathbb{E} \left\{ e^{-\beta(D_{\text{KL}}(p || \mathbf{\hat{q}}))} \right\}.
\end{aligned}
\end{equation}
The chi-squared distance is a second-order Taylor approximation of the relative entropy \cite{cover1999elements}. That is to say, for any \(p, q \in \Delta_k\), we have
\begin{equation}\label{eq.2nd order approximation}
D_{\text{KL}}(p || q) = \frac{1}{2}\chi^2(p, q) + O(\|p-q\|^2),
\end{equation}
where \(D_{\text{KL}}(p || q):= \sum_i p_i \log(p_i/q_i)\) and \(\chi^2(p,q):= \sum_i (p_i-q_i)^2/{q_i}\). Besides, for all \( p, q \in \Delta_k \), the chi-squared distance \( \chi^2(p,q) \) equals
\begin{equation}\label{eq.chi square}
\chi^2(p,q) = (p'-q')^\top C^{-1}(q') (p'-q'),
\end{equation}
where $p' := (p_1, p_2, \ldots, p_{k-1}) \in \mathbb{R}^{k-1}$ \cite{cover1999elements}.
According to the central limit theorem \cite{kwak2017central}, the empirical distribution can be approximated by a Gaussian distribution \(\hat{\mathbf{q}} \sim \mathcal{N}(\hat{q}; q, C(q))\) \cite{miller2018robust} with $q = p(\Bar{y}_t|x_t)$ being the nominal distribution. 
Defining \(C = C(q')\), by \eqref{eq.2nd order approximation} and \eqref{eq.chi square}, we have
\begin{equation}\nonumber
\begin{aligned}
& \mathbb{E} \left\{ e^{-\beta(D_{\text{KL}}(p || \mathbf{\hat{q}}))} \right\}
\\
\approx & \mathbb{E} \left\{e^{-\frac{\beta}{2}(p'-\hat{\mathbf{q}}')^{\top}C^{-1}(p'-\hat{\mathbf{q}}')}  \right\}
\\ 
= & (2\pi)^{\frac{k-1}{2}} |C/\beta|^{\frac{1}{2}} \int \mathcal{N}(p';\hat{q}', C/\beta) \mathcal{N}(\hat{q}';q', C) \, \mathrm{d} \hat{q}' 
\\
= & (2\pi)^{\frac{k-1}{2}} |C/\beta|^{\frac{1}{2}} \mathcal{N}(p';q', (\frac{1}{\beta}+ 1)C)
\\
= & (\frac{1}{1+ \beta})^{\frac{k-1}{2}} e^{-\frac{1}{2} (\frac{1}{\beta}+ 1)^{-1}(p'-q')^{\top}C^{-1}(p'-q') }
\\
\approx & (\frac{1}{1+ \beta})^{\frac{k-1}{2}} e^{-\frac{\beta}{\beta+1} D_{\text{KL}}(p || q) }
\\
= & (\frac{1}{1+ \beta})^{\frac{k-1}{2}} e^{-\frac{\beta}{\beta+1} D_{\text{KL}}(\hat{p}_{y_t} || p(\bar{y}_t|x_t)) }
\\
\propto &  p({y}_t|x_t)^{\frac{\beta}{\beta+1}}.
\end{aligned}
\end{equation}
Thus, we finally achieve $p_c(y_t|x_t) \propto p({y}_t|x_t)^{\frac{\beta}{\beta+1}}$.
\end{proof}
Based on the second order approximation of the relative entropy, this theorem provides an effective way of performing convolutional Bayesian filtering by simply transforming  transition probability and output probability using their fractional orders.

\begin{remark}
In linear Gaussian systems, the proposed approximation alters the covariances for both transition and measurement noises. Consider, for example, the nominal transition probability $p(\bar{x}_t|x_{t-1}) = \mathcal{N}(\bar{x}_t; Ax_{t-1}, Q)$; this becomes $p_e({x}_t|x_{t-1}) = \mathcal{N}(\bar{x}_t; Ax_{t-1}, (\alpha+1)/\alpha \cdot Q)$, thereby changing the covariance of the transition noise from $Q$ to $(\alpha+1)/\alpha \cdot Q$. Such a modification aligns with the guidelines of Corollary~\ref{corollary.analytical}. However, it is crucial to emphasize that while Corollary~\ref{corollary.analytical} is confined to Gaussian distributions, Theorem~\ref{theorem3.power approximation} broadens the scope to include any type of distribution. The extensive applicability of this approach is also reflected in the convolutional particle filter algorithm (see Algorithm~\ref{alg.ConvPF}), which is formulated without being limited to any specific distribution type.
\end{remark}

\begin{algorithm}[!t]
\caption{Convolutional Particle Filter with Exponential Density Rescaling}
\label{alg.ConvPF}
\begin{algorithmic}[1]
\REQUIRE Sequence of measurements $y_{1:t}$, nominal transition probability $p(x_t|x_{t-1})$, nominal output probability $p(y_t|x_t)$, number of particles $N$.
\FOR{$t = 0$ \TO $T$}
    \STATE \textbf{Prediction:}
    \FOR{$i = 1$ \TO $N$}
        \STATE 
        $x_t^{(i)} \sim p_e(x_t|x_{t-1}^{(i)})$
    \ENDFOR
    \STATE \textbf{Update:}
    \FOR{$i = 1$ \TO $N$}
        \STATE 
        $w^{i} = p_e(y_t|x_t^{(i)})$
    \ENDFOR
    \STATE Normalize weights: $w \leftarrow w / \sum(w)$
    \STATE \textbf{Resampling:}
    \STATE Select $N$ particles based on weights $w$ for new set
    \STATE Replace current particles with the new set
\ENDFOR
\ENSURE Final set of particles: $x^{(1:N)}$
\end{algorithmic}
\end{algorithm}

\subsection{Connection with Information Bottleneck Theory}
Previously, we have proved that convolutional Bayesian filtering can be approximated by exponential density rescaling technique. This section will provide a theoretical view of this technique using the information bottleneck theory.


Given the measurement data $\mathbf{y}_t$, the state $\mathbf{x}_t$ can be regarded as its compressed representation. Leveraging the information bottleneck theory \cite{tishby2000information,bialek2001predictability}, we can express the information bottleneck objective as
\begin{equation}\label{eq.IB constrained}
\begin{aligned}
q_\text{info}(x_t|y_t) =& \argmin_{q(x_t|y_t)}\left\{-I(\mathbf{x}_t, \bar{\mathbf{y}}_t) \right\} ],
\\
\text{s.t.} \;& I\left(\mathbf{x}_t, \mathbf{y}_t \right) \leq I_0,
\end{aligned}  
\end{equation}
where $I(\mathbf{x}, \mathbf{y}) = D_{\text{KL}}\left(p(x,y) \| p(x) p(y)\right)$ is defined as the mutual information between random variables $\mathbf{x}$ and $\mathbf{y}$. Here, the mutual information $I(\mathbf{x}_t, \mathbf{y}_t)$ and $I(\mathbf{x}_t, \mathbf{\bar{y}}_t)$ are defined by two joint probability distributions $p(x_t,y_t)$ and $p(x_t,\bar{y}_t)$, which can be decomposed as
\begin{equation}\nonumber
\begin{aligned}
p(x_t,y_t) &= q(x_t|y_t)p(y_t) = q(y_t|x_t)p(x_t), \\ p(x_t,\bar{y}_t) &= p(x_t|\bar{y}_t)p(\bar{y}_t).   
\end{aligned}    
\end{equation}
The goal of \eqref{eq.IB constrained} is to maximize the mutual information between the virtual measurement and its compression, the system state, while ensuring that the mutual information between the state and the actual measurement does not exceed $I_0$. The concept of the ``information bottleneck" emerges from the limitation that $I(\mathbf{x}_t, \mathbf{y}_t)$ must not exceed $I_0$, which requires compressing the information in $\mathbf{y}_t$ through a bottleneck, as depicted in Fig.~\ref{fig.information bottleneck1}.

The constrained optimization problem in \eqref{eq.IB constrained} can be transformed into an unconstrained optimization problem by using the Lagrange multiplier $1-\gamma$:
\begin{equation}\label{eq.unconstrained}
q_{\text{info}}(x_t|y_t) = \argmin_{q(x_t|y_t)} \left\{ 
-I(\mathbf{x}_t, \bar{\mathbf{y}}_t) + (1- \gamma) I(\mathbf{x}_t, \mathbf{y}_t)
\right\}.    
\end{equation}
\begin{figure}[!t]
\centering
\begin{subfigure}[b]{0.33\textwidth}
\includegraphics[width=\textwidth]{ 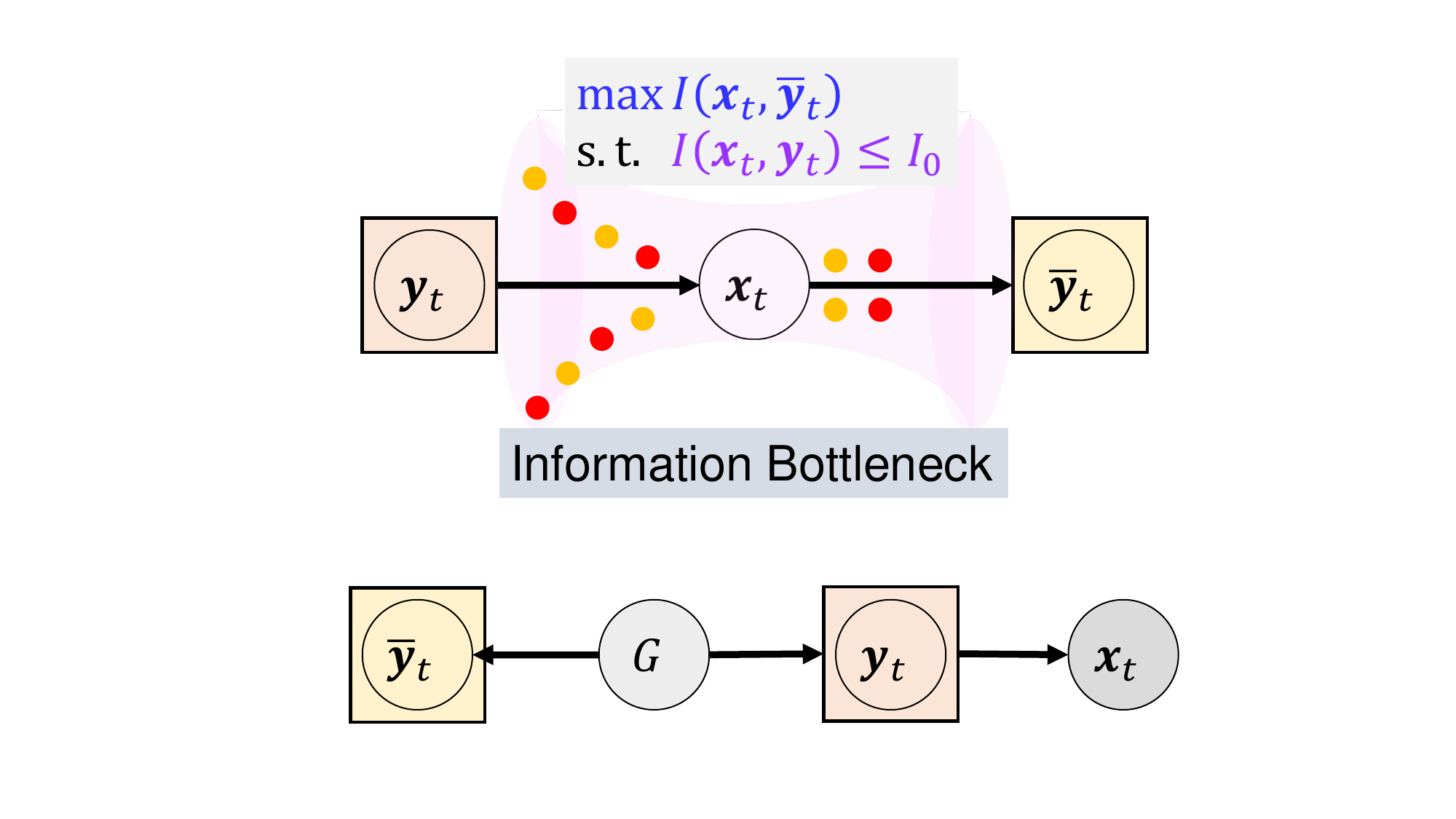}
\caption{}
\label{fig.information bottleneck1}
\end{subfigure}
\hfill
\begin{subfigure}[b]{0.33\textwidth}
\includegraphics[width=\textwidth]{ 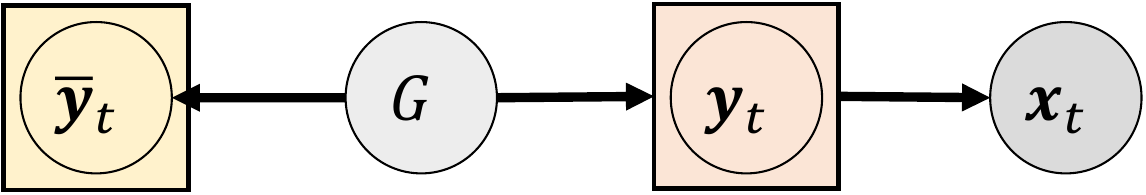}
\caption{}
\label{fig.information bottleneck2}
\end{subfigure}
\caption{(a) Illustration of information bottleneck. (b) Markov chain of the information bottleneck. Note that $G$ represents the data generator for $\mathbf{y}_t$ and $\bar{\mathbf{y}}_t$.} 
\label{fig.information bottleneck}
\end{figure}
By leveraging the Markov property (see Fig.  \ref{fig.information bottleneck2}) \cite{alemi2020variational,knoblauch2022optimization}, \eqref{eq.unconstrained} can further be  rewritten as
\begin{equation}\label{eq.IB unconstrained}
\begin{aligned}
q_{\text{info}}(x_t|y_t) 
=\argmin_{q(x_t|y_t)} \left\{I(\mathbf{x}_t,\mathbf{y}_t  | \bar{\mathbf{y}}_t) - \gamma \cdot I(\mathbf{x}_t, \mathbf{y}_t)
\right\}.    
\end{aligned}
\end{equation}
We can find an approximate upper bound of \eqref{eq.IB unconstrained}. For the first term, we have:
\begin{equation}\label{eq.IB upper bound}
\begin{aligned}
&I(\mathbf{x}_t,\mathbf{y}_t|\bar{\mathbf{y}}_t) 
\\
=& \mathbb{E}_{p(y_{t}, \bar{y}_{t})} \left\{\mathbb{E}_{q(x_t|y_{t})} \left\{\log \left(\frac{{q(\mathbf{x}_t|\mathbf{y}_t)p(\mathbf{y}_t|{\bar{\mathbf{y}}}_t)}}{p(\mathbf{x}_t|\mathbf{\bar{y}}_t)p(\mathbf{y}_t|{\bar{\mathbf{y}}}_t)} \right)
\right\}\right\} 
\\
=& \mathbb{E}_{p(y_{t}, \bar{y}_{t})} \left\{\mathbb{E}_{q(x_t|y_{t})} \left\{\log \left(\frac{{q(\mathbf{x}_t|\mathbf{y}_t)}}{p(\mathbf{x}_t|\mathbf{\bar{y}}_t)} \right) \right\}\right\} 
\\
=& \mathbb{E}_{p(y_{t})} \left\{\mathbb{E}_{q(x_t|y_{t})} \left\{\log \left(\frac{{q(\mathbf{x}_t|\mathbf{y}_t)}}{p_c(\mathbf{x}_t|y_{1:t-1})} \right) \right\}\right\}
\\
&-\mathbb{E}_{p(\bar{y}_{t})} \left\{\mathbb{E}_{p(x_t|\bar{y}_{t})} \left\{\log \left(\frac{{p(\mathbf{x}_t|\bar{\mathbf{y}}_t)}}{p_c(\mathbf{x}_t|y_{1:t-1})} \right) \right\}\right\}
\\
=& \mathbb{E}_{p(y_{t})} \left\{\mathbb{E}_{q(x_t|y_{t})} \left\{\log \left(\frac{{q(\mathbf{x}_t|\mathbf{y}_t)}}{p_c(\mathbf{x}_t|y_{1:t-1})} \right) \right\}\right\}
\\
&-\mathbb{E}_{p(\bar{y}_{t})} \left\{ D_{\mathrm{KL}}(p(x_t|\bar{\mathbf{y}}_t)\|p_c({x}_t|y_{1:t-1}))\right\}.
\end{aligned}
\end{equation}
Note that we have $q(x_t|y_t, \bar{y}_t) = q(x_t|y_t)$ due to the Markov property \cite{alemi2020variational,knoblauch2022optimization}.
Because the KL divergence is always positive, \eqref{eq.IB upper bound} can be upper bounded by
\begin{equation}\nonumber
\begin{aligned}
&I(\mathbf{x}_t,\mathbf{y}_t|\bar{\mathbf{y}}_t) 
\\
\leq & \mathbb{E}_{p(y_{t})} \left\{\mathbb{E}_{q(x_t|y_{t})} \left\{\log \left(\frac{{q(\mathbf{x}_t|\mathbf{y}_t)}}{p_c(\mathbf{x}_t|y_{1:t-1})} \right) \right\}\right\} 
\\
= & \mathbb{E}_{q(x_t|y_{t})} \left\{\mathbb{E}_{p(y_{t})} \left\{\log \left(\frac{{q(\mathbf{x}_t|\mathbf{y}_t)}}{p_c(\mathbf{x}_t|y_{1:t-1})} \right) \right\}\right\} 
\\
\approx & \mathbb{E}_{q(x_t|y_{t})} \left\{\log \left(\frac{{q(\mathbf{x}_t|y_t)}}{p_c(\mathbf{x}_t|y_{1:t-1})} \right) \right\}
\\
=& D_{\text{KL}} \left(q(x_t|y_t) \| p_c(x_t|y_{1:t-1})
\right).
\end{aligned}
\end{equation}
The approximate equality in the penultimate line is due to the substitution of expected values with sample values. Besides, the second term $I \left(\mathbf{x}_t, \mathbf{y}_t \right)$ is approximately lower bounded by 
\begin{equation}\label{eq.IB lower bound}
\begin{aligned}
&I \left(\mathbf{x}_t, \mathbf{y}_t \right) 
\\
= & H\left(\mathbf{y}_t\right) + \mathbb{E}_{q(x_t|y_{t})} \left\{\mathbb{E}_{p(y_{t})} \left\{\log{q(\mathbf{y}_t|\mathbf{x}_t)} \right\}\right\}
\\
= & H\left(\mathbf{y}_t\right) + \mathbb{E}_{q(x_t|y_{t})} \left\{\mathbb{E}_{p(y_{t})} \left\{\log{p(\mathbf{y}_t|\mathbf{x}_t)} \right\}\right\}
\\
&+ \mathbb{E}_{p(y_{t})} \left\{\mathbb{E}_{q(x_t|y_{t})} \left\{\log{q(\mathbf{y}_t|\mathbf{x}_t)} \right\}-\mathbb{E}_{q(x_t|y_{t})} \left\{\log{p(\mathbf{y}_t|\mathbf{x}_t)} \right\}\right\}
\\
= & H\left(\mathbf{y}_t\right) + \mathbb{E}_{q(x_t|y_{t})} \left\{\mathbb{E}_{p(y_{t})} \left\{\log{p(\mathbf{y}_t|\mathbf{x}_t)} \right\}\right\}
\\
&+ \mathbb{E}_{p(x_{t})} \left\{D_{\mathrm{KL}}(q(y_t|\mathbf{x}_{t}) \| p(y_t|\mathbf{x}_{t}))\right\}
\\
\geq & H\left(\mathbf{y}_t\right) + \mathbb{E}_{q(x_t|y_{t})} \left\{\mathbb{E}_{p(y_{t})} \left\{\log{p(\mathbf{y}_t|\mathbf{x}_t)} \right\}\right\}
\\
\approx & H\left(\mathbf{y}_t\right) + \mathbb{E}_{q(x_t|y_{t})} \left\{\log{p(y_t|\mathbf{x}_t)} \right\}.
\end{aligned}
\end{equation}
Here, 
$p(y_t|\mathbf{x}_t)$ is not real output probability. Instead, it is the value of the nominal output probability $p(\bar{y_t}|\mathbf{x}_t)$ at the real measurement data $y_t$.
Combining \eqref{eq.IB upper bound} and \eqref{eq.IB lower bound}, we can transform the unconstrained optimization problem in \eqref{eq.IB unconstrained} into minimizing its variational lower bound. 
\begin{equation}\label{eq.IB ELBO}
\begin{aligned}
&q_{\text{info}}(x_t|y_{t}) 
\\
=&
\argmin_{q(x_t|y_{t})} \Big\{
-\gamma
\mathbb{E}_{q(x_t|y_t)}\left\{ 
\log p(y_t|\mathbf{x}_t)
\right\} 
\\
&\quad \quad\quad\quad+
D_{\text{KL}} \left(q(x_t|y_t) \| p_c(x_t|y_{1:t-1})
\right)
\Big\}
\\
=& \argmin_{q(x_t|y_t)} \left\{\mathbb{E}_{q(x_t|y_t)} 
\left\{
\log \left(\frac{q(\mathbf{x}_t|y_t)}{p_c(\mathbf{x}_t|y_{1:t-1}) \cdot p(y_t|\mathbf{x}_t))^\gamma} \right)
\right\}
\right\}
\\
=& \frac{p_c(x_t|y_{1:t-1})p(y_t|x_t)^{\gamma}}{\int p_c(x_t|y_{1:t-1})p(y_t|x_t)^{\gamma} \,\mathrm{d}x_{t}}. 
\end{aligned}
\end{equation}
Here, the entropy of the measurement $H(\mathbf{y}_t)$ is omitted because it is a constant term regarding the optimization objective \eqref{eq.IB unconstrained}.
From \eqref{eq.IB ELBO}, it can be seen that by setting \(\gamma = \beta/(\beta + 1)\), the solution to the information bottleneck problem coincides with the update step of convolutional Bayesian filtering approximated by the exponential density rescaling technique. This relationship offers a new perspective for understanding our framework. In more details, the variable \(\gamma\) serves as a Lagrange multiplier that
balances the trade-off between reconstructing information about
measurement model and compressing representation of measurement data. As \(\beta\) increases, the compression bottleneck becomes less restrictive.  When \(\beta \to \infty\), convolutional Bayesian filtering simplifies to standard Bayesian filtering, and \(\gamma \to 1\), indicating that information is constructed without bottleneck.

\section{Simulations}\label{sec.simulations}
In this section, we evaluate our proposed framework across three benchmark systems to demonstrate its applicability to classic filtering algorithms in addressing model mismatches. We conduct 
$N=100$ Monte Carlo experiments with 
$M=40$ time steps for each simulation. In each experiment, the chosen evaluation metric is the root mean square error (RMSE), which is defined as
\begin{equation}\label{eq.ARMSE}\nonumber
\begin{aligned}
        \text{RMSE} &= \sqrt{\frac{1}{M}\sum_{i=1}^{M}\|x_i - \hat{x}_i\|^2}.
    \end{aligned}
\end{equation}
Here, $x_i, \hat{x}_i$ stand for the real and estimated state at the $i$-th step. This metric is averaged with 100 experiments for fair performance evaluation in our simulations.

\begin{figure}[!tb]
\centering
\begin{subfigure}{0.49\textwidth}
\includegraphics[width=\textwidth]{ 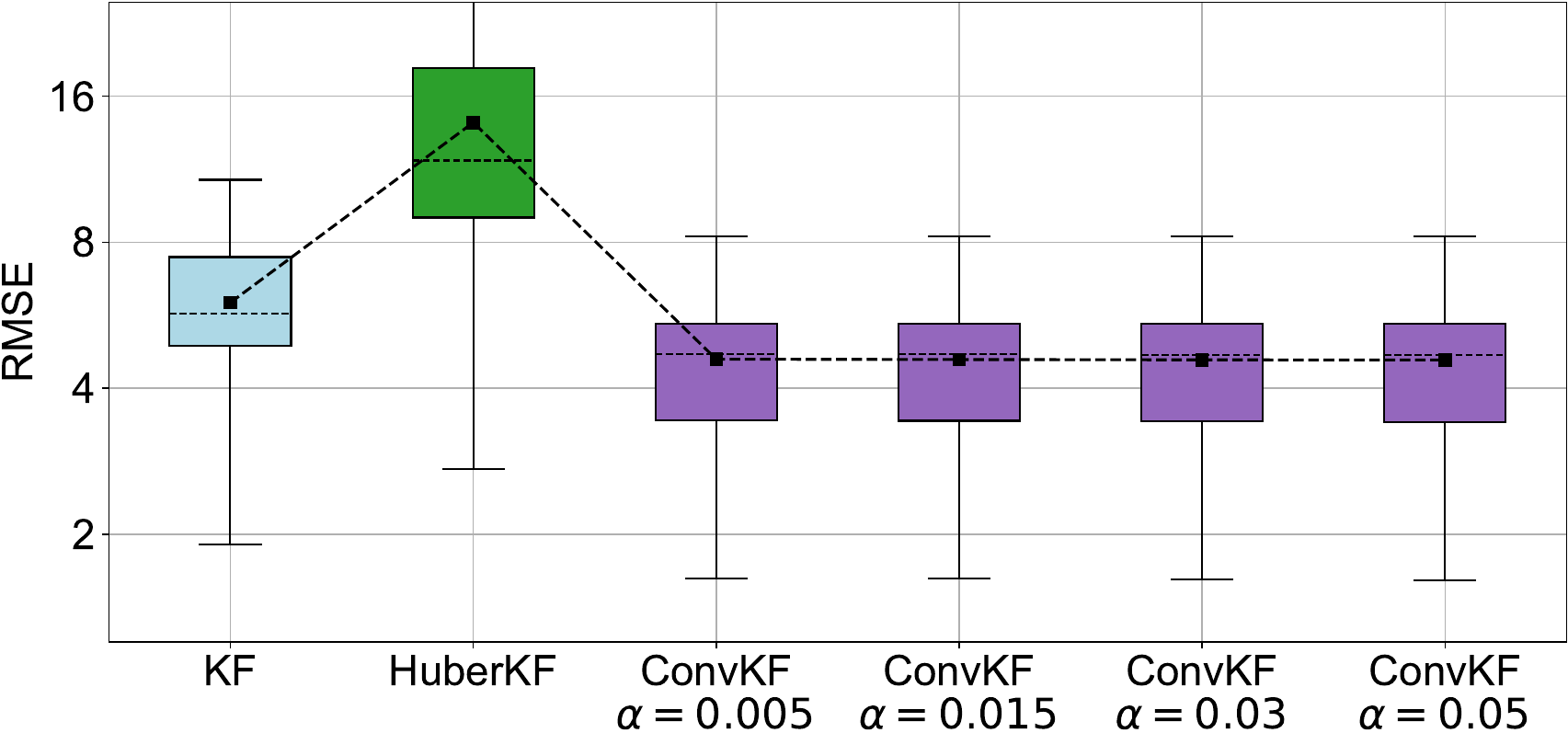}
\caption{Case A: Transition model mismatch}
\label{fig.wiener_s}
\end{subfigure}
\hfill
\\
\begin{subfigure}{0.49\textwidth}
\includegraphics[width=\textwidth]{ 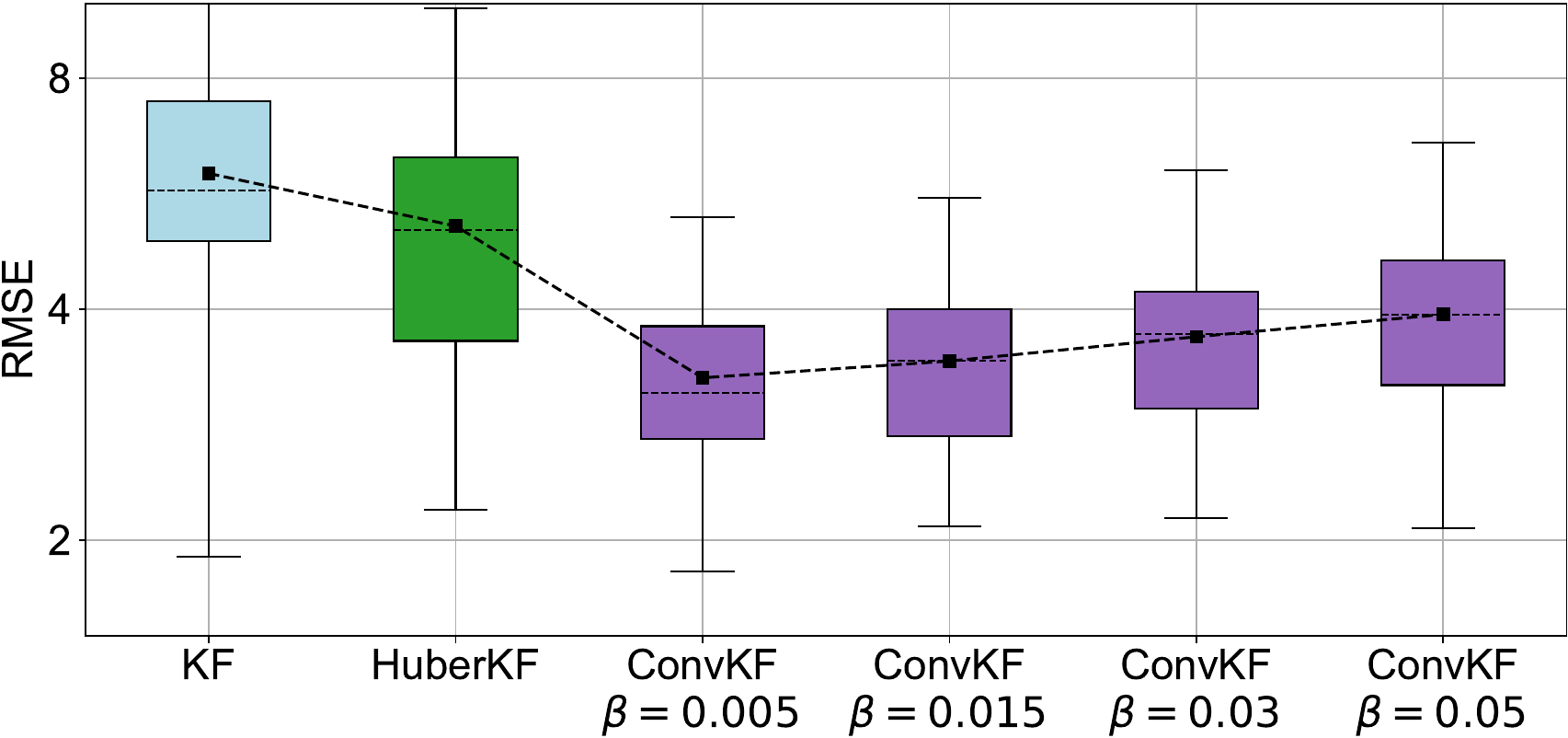}
\caption{Case B: Measurement model mismatch}
\label{fig.wiener_m}
\end{subfigure}
\hfill 
\caption{Box plot of RMSE for KF, HuberKF and ConvKF. The black square `` $\blacksquare$ " represents the average RMSE.} 
\label{fig.wiener_all_boxplot}
\end{figure}


\subsection{Linear Wiener Velocity Model}\label{subsec:Wiener}
The Wiener velocity model is a well-known standard environment in the field of target tracking, where the velocity is modeled as the Wiener process \cite{zhang2018degradation}.
The state $\mathbf{x} = \begin{bmatrix}
p_x, \, p_y, \, v_x, \, v_y    
\end{bmatrix}^{\top}$ consists of system positions $p_x, p_y$ and system velocities  $v_x, v_y$. The Wiener velocity model is described by
\begin{equation}\label{eq.wiener system}\nonumber
\begin{aligned}
\bar{\mathbf{x}}_{t+1} &= \begin{bmatrix}
1 & 0 & 0.1 & 0 \\
0 & 1 & 0 & 0.1 \\
0 & 0 & 1 & 0 \\
0 & 0 & 0 & 1
\end{bmatrix} \mathbf{x}_t + \bar{\mathbf{\xi}}_t, \\
\bar{\mathbf{y}}_t &= \begin{bmatrix}
1 & 0 & 0 & 0 \\
0 & 1 & 0 & 0
\end{bmatrix} \mathbf{x}_t + \bar{\mathbf{\zeta}}_t.
\end{aligned}
\end{equation}
\begin{figure*}[th]
\centering
\begin{subfigure}[t]{0.65\textwidth}
\includegraphics[width=\textwidth]{ 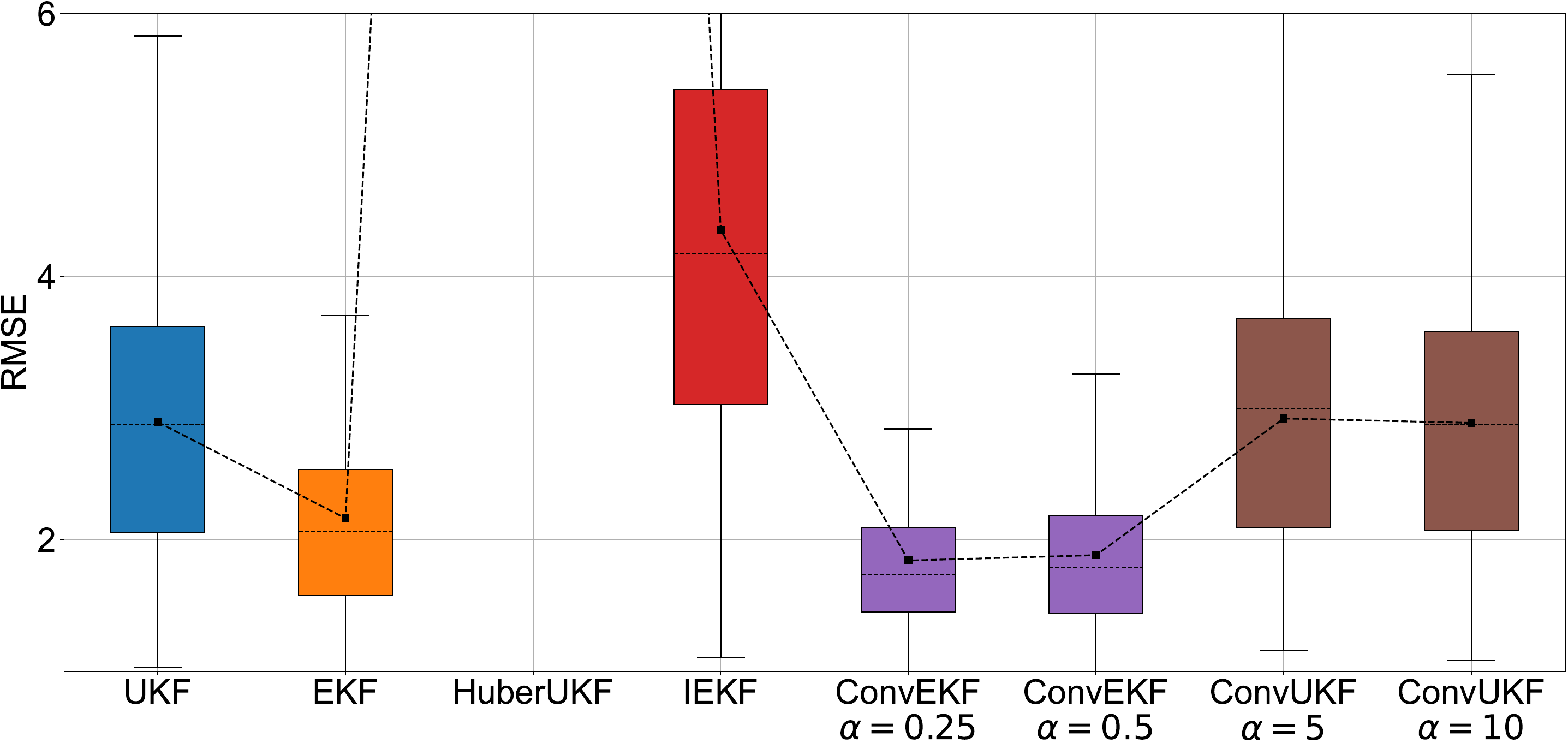}
\caption{Case A: Transition model mismatch}
\label{fig.sincos_s}
\end{subfigure}
\hfill 
\\
\begin{subfigure}[htb]{0.7\textwidth}
\includegraphics[width=\textwidth]{ 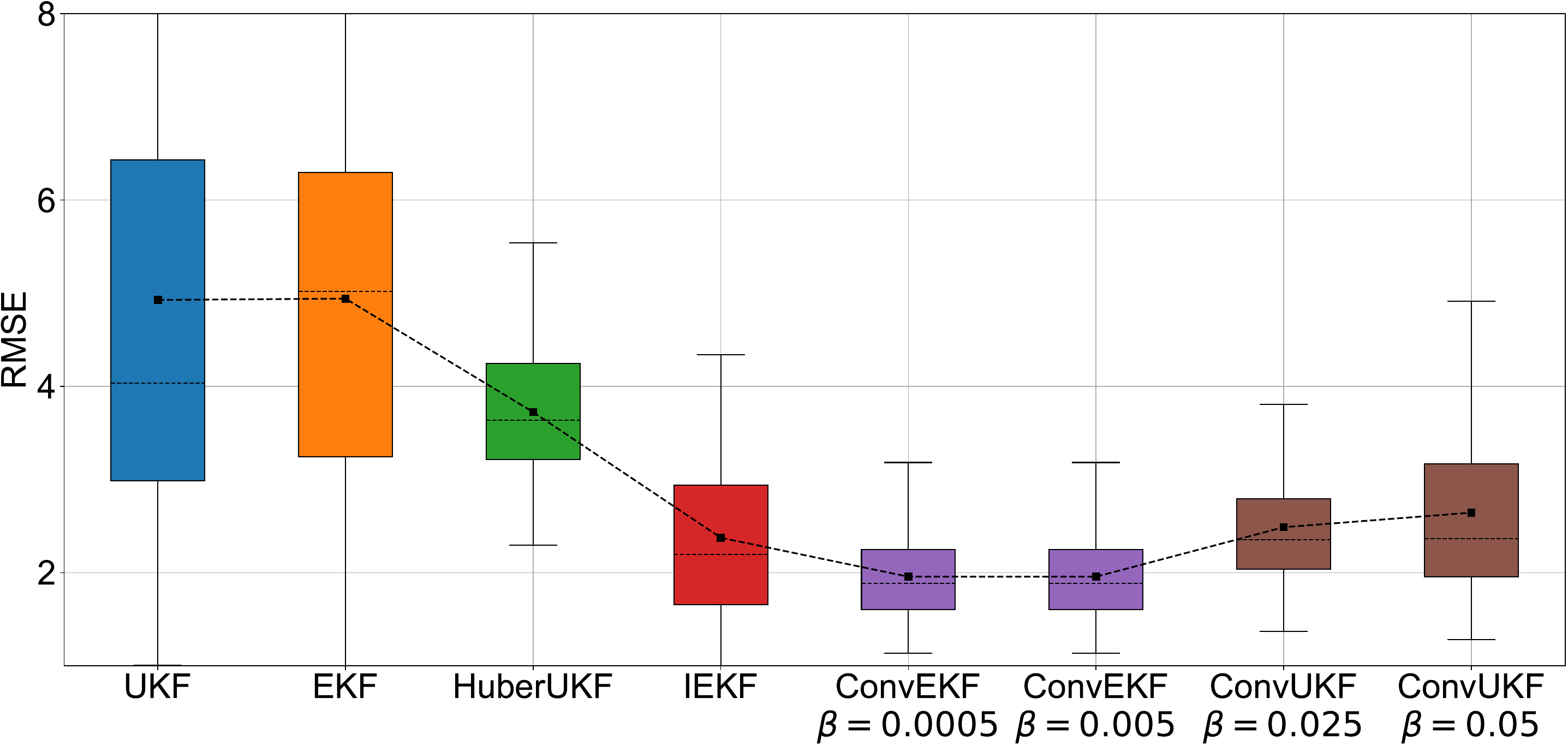}
\caption{Case B: Measurement model mismatch}
\label{fig.sincos_m}
\end{subfigure}
\caption{Box plot of RMSE for UKF, EKF, HuberUKF, IEKF, ConvEKF and ConvUKF.}
\label{fig.sincos_all_boxplot}
\end{figure*}

Here, the virtual process noise \( \bar{\xi}_t \) is modelled by \( \bar{\xi}_t \sim \mathcal{N}(0, Q) \) with covariance \( Q = \mathbb{I}_{4 \times 4} \), and virtual measurement noise satisfies \( \bar{\zeta}_t \sim \mathcal{N}(0, R) \) with \( R = \mathbb{I}_{2 \times 2} \). Additionally, the initial state \( \mathbf{x}_0 \) satisfies \( \mathbf{x}_0 \sim \mathcal{N}\left(\begin{bmatrix}
0 & 0 & 1 & 1
\end{bmatrix}^\top,\, \mathbb{I}_{4 \times 4}\right) \).

To evaluate the effectiveness of our designed filter, we consider two different cases for model mismatch, which is a common setting in existing works \cite{huber2004robust,roth2013student}:
\begin{itemize}
    \item \textbf{Case A: Transition Model Mismatch}: In the real system, the process noise is a mixture of Gaussian noises, while the measurement noise is Gaussian:
    \begin{equation}\nonumber
        \begin{aligned}
            \xi_t &\sim 0.9 \cdot \mathcal{N}(0, Q) + 0.1 \cdot \mathcal{N}(0, 100Q),\\
            \zeta_t &\sim \mathcal{N}(0, R).
        \end{aligned}
    \end{equation}

    \item \textbf{Case B: Measurement Model Mismatch}: The process noise is Gaussian, while the measurement noise is a mixture of Gaussian noises:
    \begin{equation}\nonumber
        \begin{aligned}
            \xi_t &\sim \mathcal{N}(0, Q),\\
            \zeta_t &\sim 0.9 \cdot \mathcal{N}(0, R) + 0.1 \cdot \mathcal{N}(0, 1000R).
        \end{aligned}
    \end{equation}

\end{itemize}

Our proposed convolutional Bayesian filtering framework is applied to the KF through the application of Corollary \ref{corollary.analytical}, which we have named the ConvKF. We conduct comparisons of ConvKF using various values for parameters defined in Corollary \ref{corollary.analytical}, with the standard KF and the Huber KF. Note that Huber KF is a widely-used robust method that replaces the quadratic loss in the optimization formulation of KF with the Huber loss \cite{huber2004robust}. In Fig.~\ref{fig.wiener_all_boxplot}, a box plot of the RMSE demonstrates that ConvKF outperforms the standard KF across a broad range of parameters in both cases A and B. Specifically, in case A, altering the exponential coefficient $\alpha$ from 0.005 to 0.05 results in an almost unchanged RMSE for ConvKF. In contrast, for case B, adjusting the exponential coefficient $\beta$ from 0.005 to 0.05 leads to a slight increase in the RMSE for ConvKF.


\subsection{Sequence Forecasting System}

In this subsection, we consider a popular nonlinear system used for sequence forecasting  \cite{tao2023outlier}. The state space model is given by
\begin{equation}\label{eq.sincos}\nonumber
\begin{aligned}
\bar{\mathbf{x}}_{t+1}&=x_t+\kappa_1 \cdot
\begin{bmatrix}
-1 & 0 \\
0.1 & -1 
\end{bmatrix} \mathbf{x}_t + \kappa_2 \cdot \cos(\mathbf{x}_t)+ \bar{\xi}_t,\\
\bar{\mathbf{y}}_t&=\mathbf{x}_t+\sin(\mathbf{x}_t) + \bar{\zeta}_t,
\end{aligned}
\end{equation}
where $\mathbf{x}_0\sim\mathcal{N}\left(
0 ,\,\mathbb{I}_{2\times2}\right)$. Both the constants $\kappa_1$ and $\kappa_2$ are set to $0.1$. 
We assume $Q=R=\mathbb{I}_{2\times2}$ for virtual process noise $\bar{\xi}_t \sim \mathcal{N}\left(0,\, Q\right)$ and virtual measurement noise $\bar{\zeta}_t \sim \mathcal{N}\left(0,\, R\right)$, respectively. 
We construct convolutional approaches for EKF and UKF by considering quadratic form distance metrics and an exponential distribution threshold variable, similar to  Corollary \ref{corollary.analytical}. These approaches are named the Convolutional EKF (ConvEKF) and Convolutional UKF (ConvUKF), respectively. Similar to the discussion in Subsection \ref{subsec:Wiener}, we compare our methods with the standard UKF \cite{julier1997new}, standard EKF \cite{smith1962application}, Huber UKF \cite{bing2018huber}, and Iterated EKF (IEKF) \cite{bell1993iterated}. The IEKF is a variant of the EKF that enhances linear approximation to nonlinear systems through iterative updates, thereby improving filter performance. The Huber UKF is a robust version of the UKF, which replaces the quadratic loss in the optimization of the update step with Huber loss. Our comparisons consider the following two cases:

\begin{itemize}
    \item \textbf{Case A: Transition Model Mismatch}: In the real system, the process noise is a mixture of Gaussian noises, while the measurement noise is Gaussian. Specifically, we have
    \begin{equation}\nonumber
        \begin{aligned}
            \xi_t &\sim 0.9 \cdot \mathcal{N}(0, Q) + 0.1 \cdot \mathcal{N}(0, 100Q),\\
            \zeta_t &\sim \mathcal{N}(0, R).
        \end{aligned}
    \end{equation}

    \item \textbf{Case B: Measurement Model Mismatch}: The process noise is Gaussian, while the measurement noise is a mixture of Gaussian noises. This is represented as
    \begin{equation}\nonumber
        \begin{aligned}
            \xi_t &\sim \mathcal{N}(0, Q),\\
            \zeta_t &\sim 0.9 \cdot \mathcal{N}(0, R) + 0.1 \cdot \mathcal{N}(0, 1000R).
        \end{aligned}
    \end{equation}
\end{itemize}

As demonstrated in Fig. \ref{fig.sincos_all_boxplot}, ConvEKF outperforms the other methods in both case A and case B over a wide range of parameters. Additionally, ConvUKF also shows improvements over the standard UKF, particularly in situations with measurement outliers. Notably, the Huber UKF fails in scenarios with transition model mismatch, possibly because it is designed to enhance robustness by considering the post-prediction prior in the update step, rather than directly incorporating robustness into the prediction step. 

\subsection{Isothermal Gas-phase Reactor Model}
We perform simulation on a commonly used isothermal gas-phase reactor model for state estimation \cite{schiller2023lyapunov}. This model describes the reversible chemical reaction $2A_r \rightleftharpoons B_r$. Initially, the reactor is charged with certain amounts of $A_r$ and $B_r$, but the exact composition of the original mixture remains uncertain. The state $\mathbf{x}$ includes the partial pressures, i.e.,   $\mathbf{x}=\left[\begin{matrix}P_A &P_B\end{matrix}\right]^\top$. The discrete-time version of the gas-phase reactor model with Euler method is
\begin{equation}\label{eq.isothermal gas-phase reactor}\nonumber
\begin{aligned}
\bar{\mathbf{x}}_{t+1}&=
\begin{pmatrix}
P_{A,t} +(-2k_1 \cdot P_{A,t}^2+2k_2 \cdot P_{B,t}) \cdot \mathrm{d}t \\
P_{B,t}+ (k_1 \cdot P_{A,t}^2 -k_2 \cdot P_{B,t}) \cdot \mathrm{d}t  \\
\end{pmatrix} + \bar{\xi}_t,  \\
\bar{\mathbf{y}}_t&= P_{A,t} +P_{B,t}+ \bar{\zeta}_t, 
\end{aligned}
\end{equation}
where ${\mathbf{x}}_0=[0.1,\,4.5]^\top $, $k_1=0.16$, $k_2=0.0064$, and  $\mathrm{d}t=0.1$. The virtual process noise satisfies $\bar{\xi}_t \sim \mathrm{Laplace}\left(0,\, Q\right)$ with $Q = 10^{-4} \mathbb{I}_{2\times2}$, and the virtual measurement noise satisfies $\bar{\zeta}_t \sim \mathrm{Laplace}\left(0, \, R \right)$ with $R = \mathbb{I}_{1\times1}$. For our subsequent verification, we will also set up two different simulations similar to the Section \ref{subsec:Wiener}:
\begin{itemize}
        \item \textbf{Case A, Transition model mismatch}: The measurement noise obeys Laplace distribution while the process noise is a mixture of Laplace noise, i.e.,
    \begin{equation}\nonumber
        \begin{aligned}
            \xi_t&\sim 0.9 \cdot \mathrm{Laplace}(0,Q) + 0.1 \cdot \mathrm{Laplace}(0,1000Q),\\
            \zeta_t&\sim \mathrm{Laplace}(0,R).
        \end{aligned}
    \end{equation}

        \item \textbf{Case B, Measurement model mismatch}:  The process noise obeys Laplace distribution while the measurement noise is a mixture of Laplace noise, i.e.,
    \begin{equation}\nonumber
        \begin{aligned}
            \xi_t&\sim\mathrm{Laplace}(0,Q),\\
            \zeta_t&\sim0.9 \cdot \mathrm{Laplace}(0,R)+ 0.1 \cdot \mathrm{Laplace}(0,1000R).
        \end{aligned}
    \end{equation}
\end{itemize}

We apply our proposed convolutional Bayesian filtering framework to PF, approximated using exponential density rescaling, and refer to it as ConvPF, as shown in Algorithm~\ref{alg.ConvPF}. Our method is compared with standard PF, auxiliary PF (APF) \cite{pitt1999filtering}, and student-t PF (STPF) \cite{xu2013robust}. Note that APF and STPF are two widely used robust algorithms. APF introduces an auxiliary variable to select particles based on both their weights and the likelihood of the current observation prior to the actual resampling step. This method focuses computational resources on more promising particles, enhancing the filter's performance, particularly in scenarios with tailed observation densities. On the other hand, STPF employs the Student's t distribution, which has heavier tails, making it more capable of handling extreme values or deviations from normal assumptions.

\begin{figure}[!t]
\centering
\begin{subfigure}{0.45\textwidth}
\includegraphics[width=\textwidth]{ 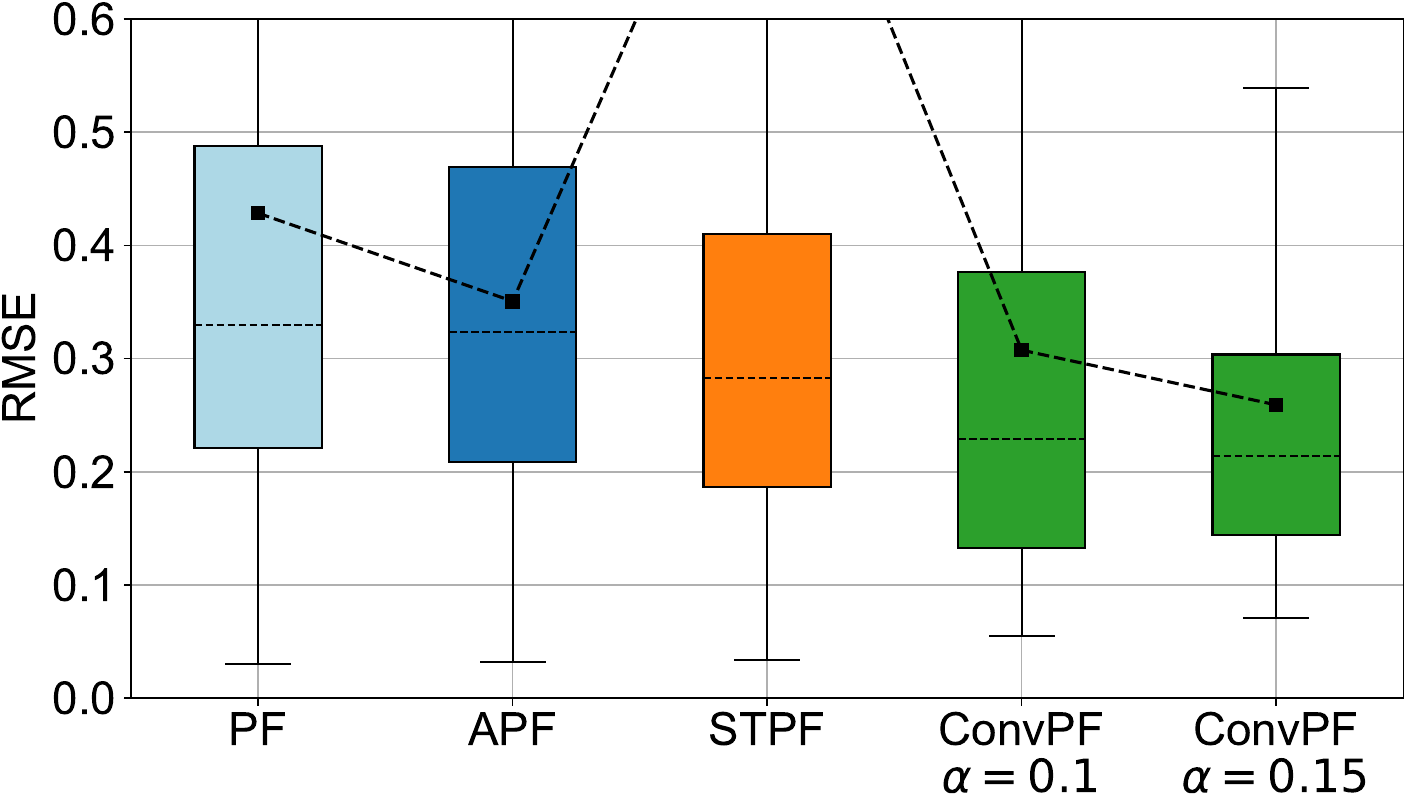}
\caption{Case A: Transition model mismatch}
\label{fig.reaction_s_laplace}
\end{subfigure}
\hfill 
\\
\begin{subfigure}{0.45\textwidth}
\includegraphics[width=\textwidth]{ 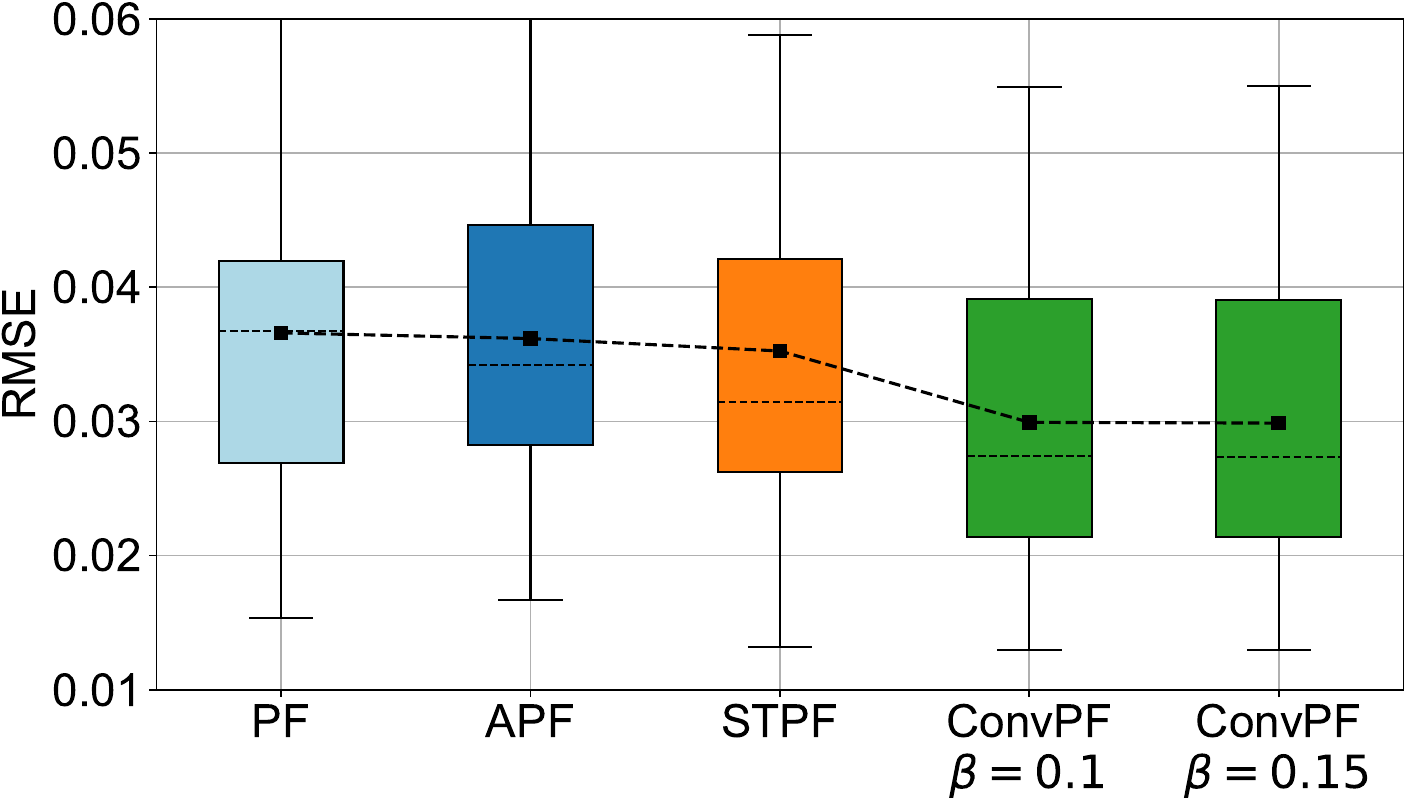}
\caption{Case B: Measurement model mismatch}
\label{fig.pf_m_laplace}
\end{subfigure}
\hfill 
\caption{Box plot of RMSE for PF, APF, STPF and ConvPF under Laplace noise conditions.} 
\label{fig.reaction_all_boxplot_laplace}
\end{figure}

As depicted in Fig.~\ref{fig.reaction_all_boxplot_laplace}, our approach yields the minimum estimation error in scenarios involving both transition (Case A), and measurement model mismatches (Case B). Although PF exhibits a marginal enhancement in RMSE over the standard PF, the improvement is not significant. STPF shows varied performance; in Case A, the STPF's median RMSE is marginally better than that of PF, yet its overall variance and average RMSE are notably higher. In Case B, the STPF does offer an improvement compared to the PF. However, our ConvPF method, with tuning parameters $\alpha$ and $\beta$, consistently outperforms the other methods.

\section{Conclusion and Discussion}\label{sec.discussion and conclusion}
This paper extends the definition of conditional probability and
introduces a convolutional Bayesian filtering framework by transforming transition and output probabilities into convolutional forms, broadening the scope of Bayesian filtering. We demonstrate that
convolutional Bayesian filtering possesses analytical forms of
convolution operation in systems with Gaussian noises.
For non-Gaussian cases, the transition and output probabilities can be effectively approximated by scaling them into fractional powers, when employing the relative entropy as the distance measure. This leads to an enhanced version of the Kalman filter, which achieves robustness through simple modifications to the noise covariance matrix, while still preserving the conjugate nature of Gaussian distributions.
The practical efficacy of convolutional Bayesian filtering is demonstrated through its application to various common filtering algorithms, including the Kalman filter, extended Kalman filter, unscented Kalman filter, and particle filter.

In this paper, our primary focus is the generalization of Bayesian filtering theory to a convolutional form. Bayesian filtering undeniably forms the foundation of optimal filtering theory for discrete-time systems, highlighting the significance and applicability of our extension. Nevertheless, it's also crucial to acknowledge the distinctive aspects of filtering theory for continuous-time systems. In these systems, the conditional density function of states typically derives from numerical solutions of Kusher’s or
Duncan-Mortensen-Zakai’s equations \cite{zakai1969optimal,mortensen1966optimal}, rather than Bayes' law. A notable advancement in this domain is the Yau-Yau method \cite{yau2000real, yau2008real}, which is rigorously proven to converge to a global solution (a type of convergence otherwise only seen in particle filters in discrete-time systems) and can be pre-computed offline, facilitating real-time applications \cite{chen2018real,luo2013complete}. While we do not explore how to apply our approach to continuous-time systems in this paper, such an extension is a compelling future research avenue.

\bibliographystyle{plain}
\bibliography{autosam}

\end{document}